\documentclass[final]{elsarticle}

\usepackage[scaled=0.85]{FiraMono}

\usepackage{amsmath,amsfonts,bm}

\def\eqref#1{equation~\ref{#1}}

\def\1{\bm{1}}

\def\va{{\bm{a}}}

\def\vc{{\bm{c}}}

\def\vf{{\bm{f}}}

\def\vh{{\bm{h}}}

\def\vk{{\bm{k}}}

\def\vq{{\bm{q}}}

\def\vu{{\bm{u}}}
\def\vv{{\bm{v}}}

\def\mD{{\bm{D}}}

\def\mG{{\bm{G}}}

\def\mR{{\bm{R}}}

\def\mW{{\bm{W}}}

\DeclareMathAlphabet{\mathsfit}{\encodingdefault}{\sfdefault}{m}{sl}
\SetMathAlphabet{\mathsfit}{bold}{\encodingdefault}{\sfdefault}{bx}{n}
\newcommand{\tens}[1]{\bm{\mathsfit{#1}}}

\def\tD{{\tens{D}}}

\def\tR{{\tens{R}}}

\def\tT{{\tens{T}}}

\DeclareMathOperator*{\argmin}{arg\,min}

\usepackage{amsthm,amssymb,amsmath}

\usepackage{xurl}            
\usepackage{booktabs}       
\usepackage{nicefrac}       
\usepackage{microtype}      
\usepackage{xcolor}         
\usepackage{algorithm}
\usepackage{algorithmic}
\usepackage{mathrsfs}
\usepackage{mathtools}

\usepackage{booktabs}
\usepackage{graphicx}
\graphicspath{{./}{../}{./figures/}}
\usepackage{caption}
\usepackage{subfigure}
\usepackage{algorithm}
\usepackage{multirow}
\usepackage{xcolor}
\usepackage{enumitem}
\usepackage{diagbox}

\usepackage{geometry}

\usepackage{hyperref}       
\hypersetup{
    colorlinks,
}

\theoremstyle{plain}
\newtheorem{theorem}{Theorem}[section]
\newtheorem{proposition}[theorem]{Proposition}

\theoremstyle{definition}

\theoremstyle{remark}

\begin{document}

\hypersetup{
	plainpages=false,
    linktocpage=true, 
    linkcolor={red!70!black},
    citecolor={green!60!black},
    urlcolor={cyan!50!black}
} 

\begin{frontmatter}
\title{Mitigating spectral bias for the multiscale operator learning
}

\author[inst1]{Xinliang Liu}
\affiliation[inst1]{organization={Computer, Electrical and Mathematical Science and Engineering Division, King Abdullah University of Science and Technology
},
            city={Thuwal},
            postcode={23955}, 
            country={Saudi Arabia}}

\author[inst2]{Bo Xu}
\affiliation[inst2]{organization={School of Mathematical Sciences, Shanghai Jiao Tong University},
            city={Shanghai},
            postcode={200240}, 
            country={China}}

\author[inst3]{Shuhao Cao}
\affiliation[inst3]{organization={School of Science and Engineering, University of Missouri-Kansas City},
            city={Kansas City},
            postcode={64110}, 
            state={MO},
            country={United States}}

\author[inst4]{Lei Zhang}
\affiliation[inst4]{organization={School of Mathematical Sciences, Institute of Natural Sciences and MOE-LSC, Shanghai Jiao Tong University},
            city={Shanghai},
            postcode={200240}, 
            country={China}}

\begin{abstract}
Neural operators have emerged as a powerful tool for learning the mapping between infinite-dimensional parameter and solution spaces of partial differential equations (PDEs). In this work, we focus on multiscale PDEs that have important applications such as reservoir modeling and turbulence prediction. We demonstrate that for such PDEs, the spectral bias towards low-frequency components presents a significant challenge for existing neural operators. To address this challenge, we propose a { hierarchical attention neural operator (HANO) inspired by the hierarchical matrix approach}. HANO features a scale-adaptive interaction range and self-attentions over a hierarchy of levels, enabling nested feature computation with controllable linear cost and encoding/decoding of multiscale solution space. We also incorporate an empirical $H^1$ loss function to enhance the learning of high-frequency components. Our numerical experiments demonstrate that HANO outperforms state-of-the-art (SOTA) methods for representative multiscale problems. 
\end{abstract}

\begin{keyword}
partial differential equations \sep operator learning \sep transformer \sep multiscale PDE
\end{keyword}

\end{frontmatter}

\section{Introduction}
\label{sec:introduction}


In recent years, operator learning methods have emerged as powerful tools for computing parameter-to-solution maps of partial differential equations (PDEs). In this paper, we focus on the operator learning for multiscale PDEs (MsPDEs) that encompass multiple temporal/spatial scales. MsPDE models arise in applications involving heterogeneous and random media, and are crucial for predicting complex phenomena such as reservoir modeling, atmospheric and ocean circulation, and high-frequency scattering. Important prototypical examples include multiscale elliptic partial differential equations, where the diffusion coefficients vary rapidly. The coefficient can be potentially rapidly oscillatory, have high contrast ratio, or even bear a continuum of non-separable scales.

MsPDEs, even with fixed parameters, present great challenges for classical numerical methods \cite{Branets2009}, as their computational cost typically scales inversely proportional to the finest scale $\varepsilon$ of the problem. To overcome this issue, multiscale solvers have been developed by incorporating microscopic information to achieve computational cost independent of $\varepsilon$. One such technique is \emph{numerical homogenization} \citep{Engquist2008,Hou1999,eh09,Efendiev2013,chung2016adaptive,chung2023multiscale}, which identifies low-dimensional approximation spaces adapted to the corresponding multiscale operator. Similarly, fast solvers like \emph{multilevel/multigrid methods} \citep{Hackbusch1985,XuZikatanov:2017} and \emph{wavelet-based multiresolution methods} \citep{Brewster1995,Beylkin1998} may face limitations when applied to multiscale PDEs \citep{Branets2009}, while multilevel methods based on numerical homogenization techniques, such as Gamblets \citep{OwhadiMultigrid:2017}, have emerged as a way to discover scalable multilevel algorithms and operator-adapted wavelets for multiscale PDEs. Low-rank decomposition-based methods are another popular approach to exploit the low-dimensional nature of MsPDEs. Notable example include the fast multipole method \cite{Greengard1987}, hierarchical matrices ($\mathscr{H}$ and $\mathscr{H}^2$ matrices) \cite{Hackbusch2002}, and hierarchical interpolative factorization \cite{Ho2016}. These methods can achieve (near-)linear scaling and high computational efficiency by exploiting the low-rank approximation of the (elliptic) Green's function \cite{Bebendorf2005}.

Neural operators, unlike traditional solvers that operate with fixed parameters, are capable of handling a range of input parameters, making them promising for data-driven forward and inverse solving of PDE problems. Pioneering work in operator learning methods include \cite{ZhuZabaras:2018,Fan2019,fan2019multiscale,Khoo2020}. Nevertheless, they are limited to problems with fixed discretization sizes. Recently, infinite-dimensional operator learning has been studied, which learns the solution operator (mapping) between infinite-dimensional Banach spaces for PDEs. Most notably, the Deep Operator Network (DeepONet) \cite{lu2021learning} was proposed as a pioneering model to leverage deep neural networks' universal approximation for operators \cite{chen1995universal}. Taking advantage of the Fast Fourier Transform (FFT), Fourier Neural Operator (FNO) \cite{li2020fourier} constructs a learnable parametrized kernel in the frequency domain to render the convolutions in the solution operator more efficient. Other developments include the multiwavelet extension of FNO \cite{gupta2021multiwavelet}, Message-Passing Neural Operators \cite{brandstetter2022message}, dimension reduction in the latent space \cite{seidman2022nomad}, Gaussian Processes \cite{chen2021solving}, Clifford algebra-inspired neural layers \cite{brandstetter2023clifford}, and Dilated convolutional residual network \cite{stachenfeld2022learned}.

Attention neural architectures, popularized by the Transformer deep neural network \cite{vaswani2017attention}, have emerged as universal backbones in Deep Learning. These architectures serve as the foundation for numerous state-of-the-art models, including GPT \cite{brown2020language}, Vision Transformer (ViT) \cite{dosovitskiy2020image}, and Diffusion models \cite{ho2020denoising,rombach2022high}. More recently, Transformers have been studied and become increasingly popular in PDE operator learning problems, e.g., in \cite{cao2021choose,geneva2022transformers,kissas2022learning,li2023transformer,hao2023gnot,pmlr-v202-de-oliveira-fonseca23a,xiao2023improved} and many others. There are several advantages in the attention architectures. Attention can be viewed as a parametrized instance-dependent kernel integral to learn the ``basis'' \cite{cao2021choose} similar to those in the numerical homogenization; see also the exposition featured in neural operators \cite{2023JMLRNeural}. This layerwise latent updating resembles the learned ``basis'' in DeepONet \cite{hao2023gnot}, or frame \cite{bartolucci2023neural}. It is flexible to encode the non-uniform geometries in the latent space \cite{liu2023nuno}. In \cite{ovadia2023ditto,ovadia2023vito}, advanced Transformer architectures (ViT) and Diffusion models are combined with the neural operator framework. In \cite{2023HemmasianFarimani}, Transformers are combined with reduced-order modeling to accelerate the fluid simulation for turbulent flows. In \cite{2023NIPSLiShuFarimani},  { tensor decomposition techniques are employed to enhance the efficiency of attention mechanisms in solving high-dimensional partial differential equation (PDE) problems.} 

Among these data-driven operator learning models, under certain circumstances, the numerical results could sometimes overtake classical numerical methods in terms of efficiency or even in accuracy.
For instance, full wave inversion is considered in \cite{zhu2023fourier} with the fusion model of FNO and DeepONet (Fourier-DeepONet); direct methods-inspired DNNs are applied to the boundary value Calder\'{o}n problems achieve much more accurate reconstruction with the help of data \cite{guo2021construct,guo2023transformer,guo2023learn}; in \cite{mizera2023scattering}, the capacity of FNO to jump significantly large time steps for spatialtemporal PDEs is exploited to infer the wave packet scattering in quantum physics and achieves magnitudes more efficient result than traditional implicit Euler marching scheme. \cite{2022LiFarimani} exploits the capacity of graph neural networks to accelerate particle-based simulations. \cite{zhang2022hybrid} investigates the integration of the neural operator DeepONet with classical relaxation techniques, resulting in a hybrid iterative approach. Meanwhile, Wu et al.~\cite{wu2024capturing} introduce an asymptotic-preserving convolutional DeepONet designed to capture the diffusive characteristics of multiscale linear transport equations.

For multiscale PDEs, operator learning methods can be viewed as an advancement beyond multiscale solvers such as numerical homogenization. Operator learning methods have two key advantages: (1) They can be applied to an ensemble of coefficients/parameters, rather than a single set of coefficients, which allows the methods to capture the stochastic behaviors of the coefficients; (2) The decoder in the operator learning framework can be interpreted as a data-driven basis reduction procedure from the latent space (high-dimensional) that approximates the solution data manifold (often lower-dimensional) of the underlying PDEs. This procedure offers automated data adaptation to the coefficients, enabling accurate representations of the solutions' distributions. In contrast, numerical homogenization typically relies on a priori bases that are not adapted to the ensemble of coefficients. In this regard, the operator learning approach has the potential to yield more accurate reduced-order models for multiscale PDEs with parametric/random coefficients. 

However, for multiscale problems, current operator learning methods have primarily focused on representing the smooth parts of the solution space. This results in the so-called ``spectral bias'', leaving the resolution of intrinsic multiscale features as a significant challenge. The spectral bias, also known as the frequency principle \cite{rahaman2018spectral,ronen2019the,xu2020frequency}, states that deep neural networks (DNNs) often struggle to learn high-frequency components of functions that vary at multiple scales.
In this regard, Fourier or wavelet-based methods are not always effective for MsPDEs, even for fixed parameters. Neural operators tend to fit low-frequency components faster than high-frequency ones, limiting their ability to accurately capture fine details. When the elliptic coefficients are smooth, the coefficient to solution map can be well resolved by the FNO parameterization \cite{li2020fourier}. Nevertheless, existing neural operators have difficulty learning high-frequency components of multiscale PDEs, as is shown in Figure \ref{fig:gamblet_fig} and detailed in Section \ref{sec:experiments}.
 While the universal approximation theorems can be proven for FNO type models (see e.g., \cite{kovachki2021universal}), achieving a meaningful decay rate requires ``extra smoothness'', which may be absent or lead to large constants for MsPDEs. For FNO, this issue was partially addressed in \cite{zhao2022incremental}, yet the approach there needs an ad-hoc manual tweak on the weights for the modes chosen. 
 
 We note that for fixed parameter MsPDEs, In recent years, there has been increasing exploration of neural network methods for solving multiscale PDEs despite the spectral bias or frequency principle \citep{rahaman2018spectral, ronen2019the, xu2020frequency} indicating that deep neural networks (DNNs) often struggle to effectively capture high-frequency components of functions. Specifically designed neural solvers \cite{Li_2020, WANG2021113938, li2021subspace} have been developed to mitigate the spectral bias and accurately solve multiscale PDEs with fixed parameters.

Motivated by the aforementioned challenges, we investigate the spectral bias present in existing neural operators. Inspired by conventional multilevel methods and numerical homogenization, we propose a new \textbf{H}ierarchical \textbf{A}ttention \textbf{N}eural \textbf{O}perator (HANO) architecture to mitigate it for multiscale operator learning. We also test our model on standard operator learning benchmarks including the Navier-Stokes equation in the turbulent regime, and the Helmholtz equation in the high wave number regime.
Our main contributions can be summarized as follows:


\begin{itemize}
\item We introduce HANO, 
 that decomposes input-output mapping into hierarchical levels in an automated fashion, and enables nested feature updates through hierarchical local aggregation of self-attentions with a controllable linear computational cost.

\item We use an empirical $H^1$ loss function to further reduce the spectral bias and improve the ability to capture the oscillatory features of the multiscale solution space;

\item We investigate the spectral bias in the existing neural operators and empirically verify that HANO is able to mitigate the spectral bias. HANO substantially improves accuracy, particularly for approximating derivatives, and generalization for multiscale tasks, compared with state-of-the-art neural operators and efficient attention/transformers. 
\end{itemize}

\begin{figure}[htbp]
    \centering
    \subfigure[\scriptsize multiscale trigonometric coefficient,]{\includegraphics[width=0.23\textwidth]{  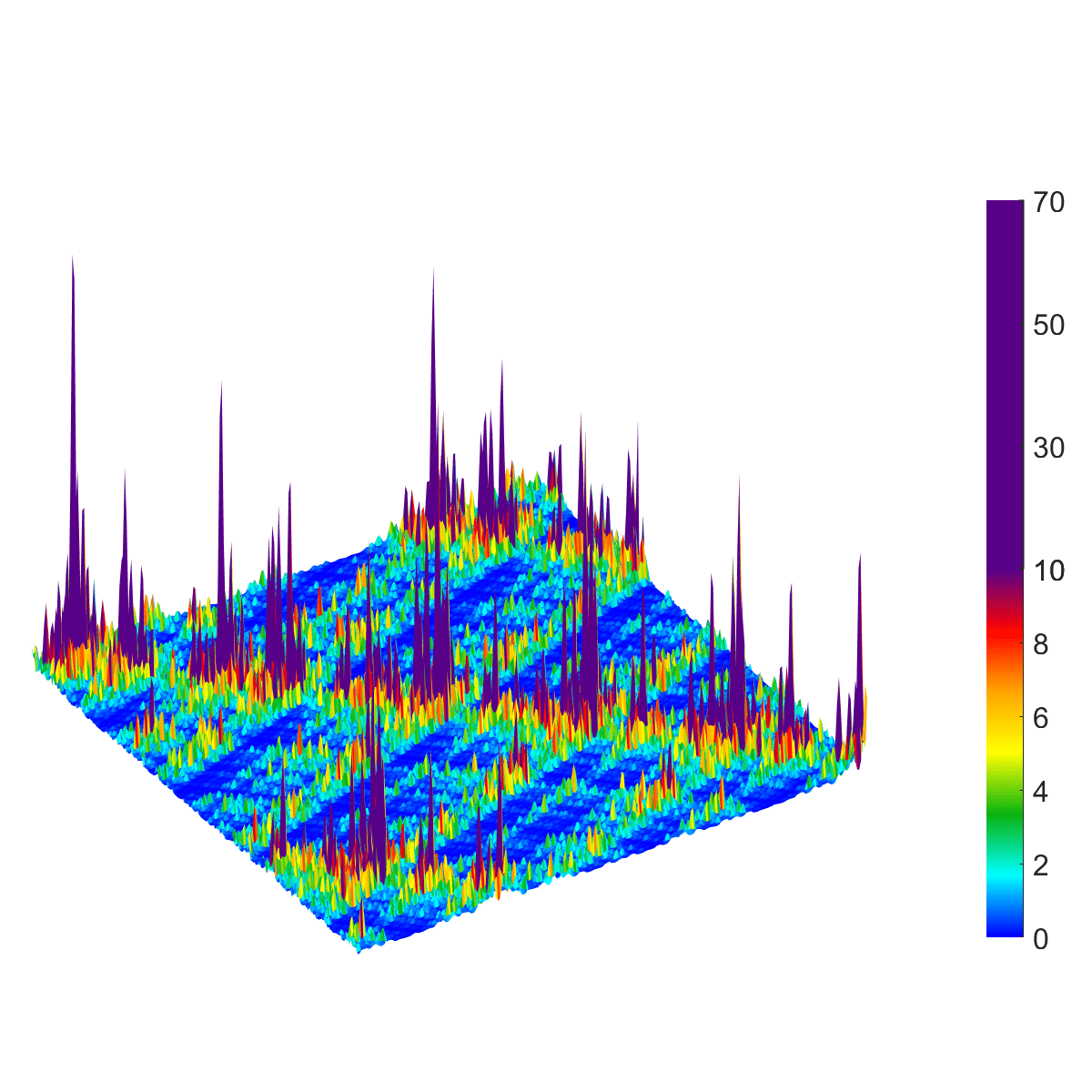}}
    \subfigure[\scriptsize slices of the derivatives $\frac{\partial u}{\partial y}$ at $x=0$,]{\includegraphics[width=0.23\textwidth]{  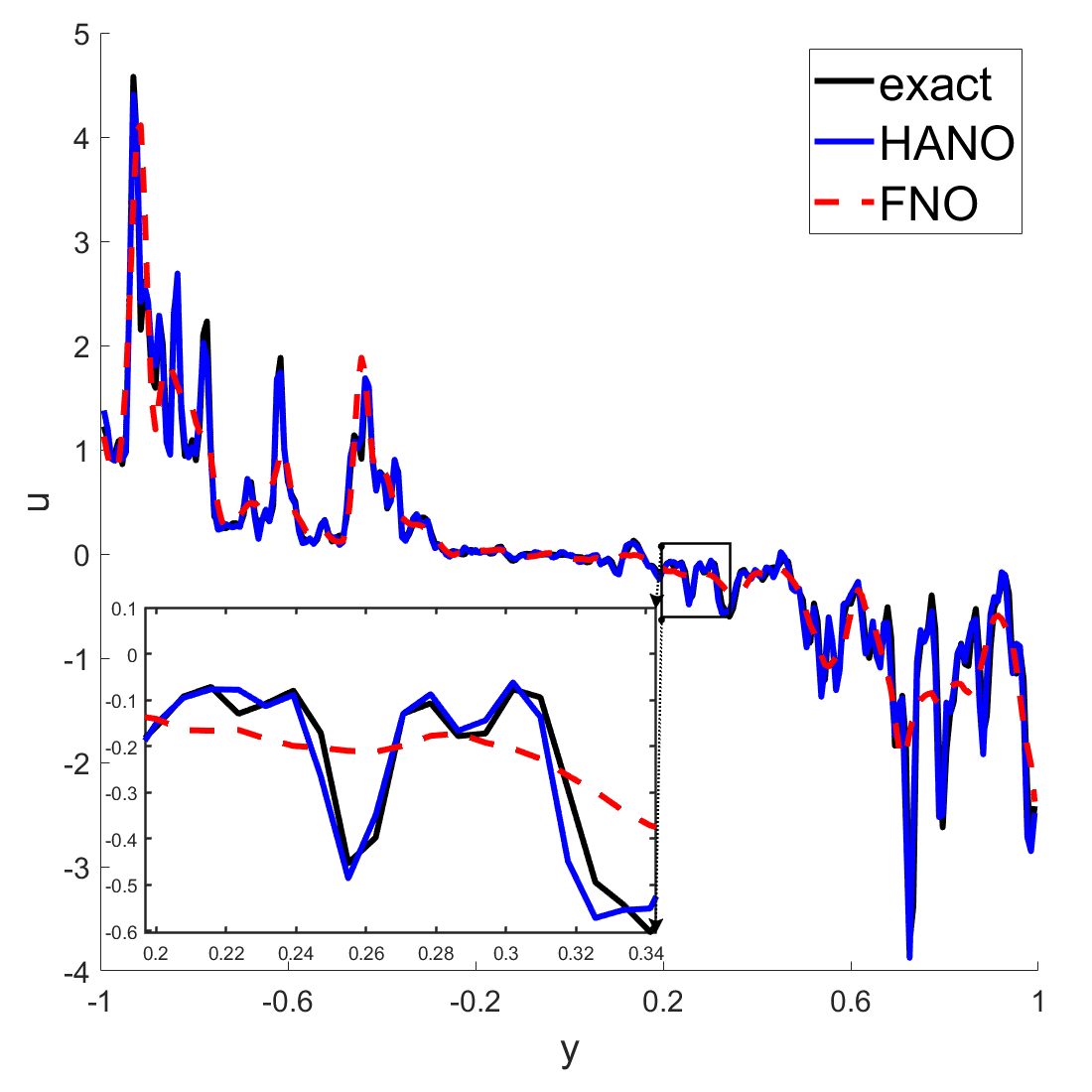}}
    \subfigure[\scriptsize absolute error spectrum of HANO in $\log_{10}$ scale]{\includegraphics[width=0.23\textwidth]{  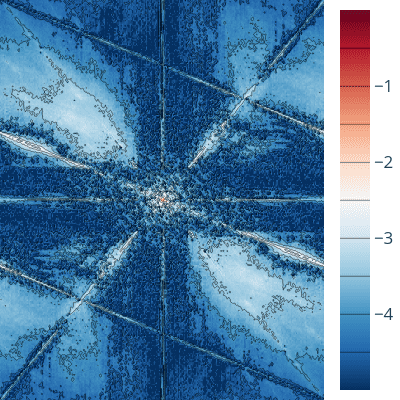}}
    \subfigure[\scriptsize absolute error spectrum of FNO in $\log_{10}$ scale.]{\includegraphics[width=0.23\textwidth]{  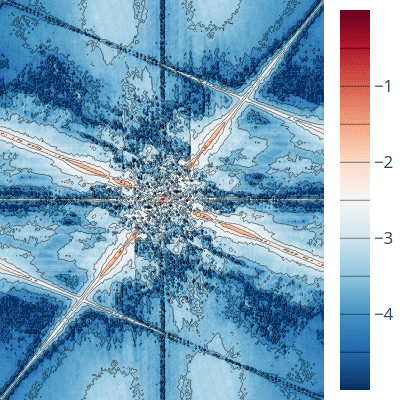}}
    \caption{We illustrate the effectiveness of the HANO scheme on the challenging multiscale trigonometric benchmark, with the coefficients and corresponding solution derivative shown in (a) and (b), see Appendix \ref{sec:additonaltests:multiscaletrignometric} for problem description. We notice that HANO can capture the solution derivatives more accurately, whereas FNO only captures their averaged or homogenized behavior. In (c) and (d), we analyze the  error by decomposing it into the frequency domain $[-256\pi, 256\pi]^2$ and plotting the absolute error spectrum. This shows the spectral bias in the existing state-of-the-art model, and also our method achieves superior performance in predicting fine-scale features, especially accurately capturing derivatives. We refer readers to Figure \ref{fig:spectral_bias} in Section \ref{sec:experiments:darcy} and Figures 
    \ref{fig:spectral_bias}, \ref{fig:gamblet_fig2},  \ref{fig:gamblet_error_freqerror}.  }
    \label{fig:gamblet_fig}
\end{figure}

\section{Methods}  
\label{sec:methods}

\def\lspace{\mathcal{V}\left(D ; \mathbb{R}^{\mathcal{C}}\right)}

In this section, to address the spectral bias for multiscale operator learning, and motivated by the remarkable performance of attention-based models \cite{vaswani2017attention,liu2021swin} in computer vision and natural language processing tasks, as well as the effectiveness of hierarchical matrix approach \cite{Hackbusch2002} for multiscale problems, we propose the Hierarchical Attention Neural Operator (HANO) model. 

\subsection{Operator Learning Problem}
We follow the setup in \cite{li2020fourier, lu2021learning} to approximate the operator $\mathcal{S}: a \mapsto u:=\mathcal{S}(a)$, with the input/parameter $a\in\mathcal{A}$ drawn from a distribution $\mu$ and the corresponding output/solution $u\in\mathcal{U}$, where $\mathcal{A}$ and $\mathcal{U}$ are infinite-dimensional Banach spaces, respectively. Our aim is to learn the operator $\mathcal{S}$ from a collection of finitely observed input-output pairs through a parametric map $\mathcal{N}: \mathcal{A} \times \Theta \rightarrow \mathcal{U}$ and a loss functional $\mathcal{L}: \mathcal{U} \times \mathcal{U} \rightarrow \mathbb{R}$, such that the optimal parameter
$$
\theta^* = \argmin _{\theta \in \Theta} \mathbb{E}_{a \sim \mu}\left[\mathcal{L}\left(\mathcal{N}(a, \theta), \mathcal{S}(a)\right)\right].
$$



\subsubsection{Hierarchical Discretization}

To develop a hierarchical attention, first we assume that there is a hierarchical discretization of the spatial domain $D$.
For an input feature map that is defined on a partition of $D$, for example, of resolution $8\times 8$ patches, we define $\mathcal{I}^{(3)}:=\{i=(i_1,i_2,i_3)\, | i_1,i_2,i_3 \in \{0,1,2,3\}\}$ as the finest level index set, in which each index $i$ corresponds to a patch \emph{token} characterized by a feature vector $\vf_i^{(3)} \in \mathbb{R}^{\mathcal{C}^{(3)}}$. For a token $i=(i_1,i_2,i_3)$, its parent token $j=\left(i_1, i_2\right)$ aggregates finer level tokens (e.g., $(1,1)$ is the parent of $(1,1,0), (1,1,1), (1,1,2), (1,1,3)$ in Figure \ref{fig:quadtree}), { characterized by a feature vector $\vf_{j}^{(2)} \in \mathbb{R}^{\mathcal{C}^{(2)}}$}. We postpone describing the aggregation scheme in the following paragraph. { In general, we write   $\mathcal{I}^{(m)}:=\{i=(i_1,i_2,...,i_m)\, | i_\ell \in \{0,1,2,3\} \text{ for } \ell=1,...,m\}$ as the index set of $m$-th level tokens, and $\mathcal{I}^{(r)}$ for $r\geq 1$ denotes the index set of the finest level tokens.} Note that the hierarchy is not restricted to the quadtree setting.
\begin{figure}[htb] 
\centering
        \includegraphics[width=.4\linewidth]{  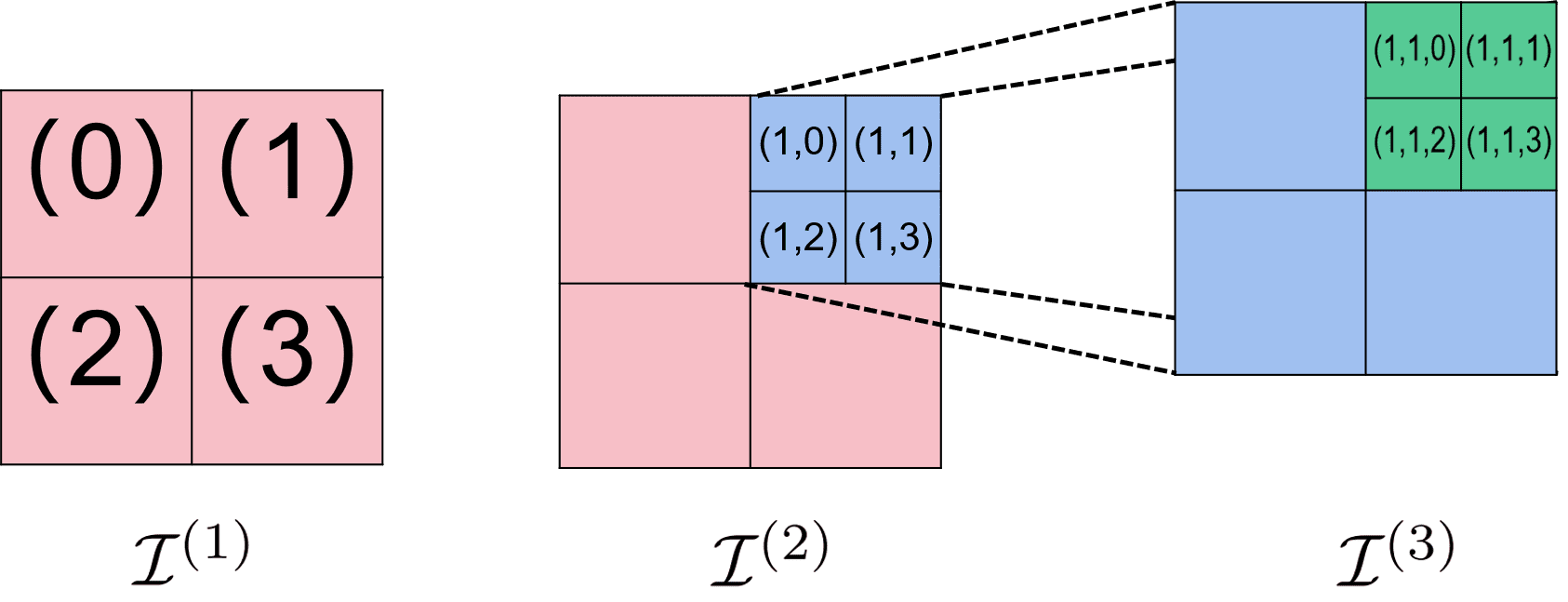}
        \caption{ \small Hierarchical discretization and index tree. The 2D unit square is discretized hierarchically into three levels with corresponding index sets $\mathcal{I}^{(1)}, \mathcal{I}^{(2)}$, and $\mathcal{I}^{(3)}$. To illustrate, $(1)^{(1,2)}$ represents the second level child nodes of node $(1)$ and is defined as $(1)^{(1,2)}=\{(1,0),(1,1),(1,2),(1,3)\}$. }
        \label{fig:quadtree} 
\end{figure}

\subsection{Vanilla Attention Mechanism}

In this section, we first revisit the vanilla scaled dot-product attention mechanism for a single-level discretization. {Without loss of generality, for example, we consider the finest level tokens, which is \( \vf_i^{(r)} \in \mathbb{R}^{\mathcal{C}^{(r)}} \), are indexed by \( i\in\mathcal{I}^{(r)} \). The token aggregation formula on this level can then be expressed as:}

\begin{equation}
\textbf{atten}: \vh_i^{(r)} = \sum_{j\in \mathcal{I}^{(r)}} \mathcal{G}(\vq_i^{(r)}, \vk_j^{(r)})\, \vv_j^{(r)},
\label{eqn:vanilla_attn}
\end{equation}
where $\boldsymbol{q}_i^{(r)} = \mW^{Q}\vf_i^{(r)}$, $\boldsymbol{k}_i^{(r)} = \mW^{K}\vf_i^{(r)}$, $\boldsymbol{v}_i^{(r)} = \mW^{V} \vf_i^{(r)}$, and $\mW^{Q}$, $\mW^{K}$, $\mW^{V} \in \mathbb{R}^{\mathcal{C}^{(r)}\times \mathcal{C}^{(r)}}$ are learnable matrices. Here, for simplicity, we use the function $\mathcal{G}$ to represent a pairwise interaction between queries and keys in the self-attention mechanism. Note that in the conventional self-attention mechanism \cite{vaswani2017attention}, the pairwise interaction potential is defined as follows:
\begin{equation}
    \mathcal{G}(\vq_i^{(r)}, \vk_j^{(r)}):=\exp(\vq_i^{(r)} \cdot \vk_j^{(r)}/ \sqrt{\mathcal{C}^{(r)}})
\end{equation}  
and further normalized to have row sum $1$, i.e., the $\text{softmax}$ function is applied row-wise to the matrix whose $(i,j)$-entry is $\vq_i^{(r)} \cdot \vk_j^{(r)}$. Note that the $1/\sqrt{\mathcal{C}^{(r)}}$ factor is optional and can be set to $1$ instead. {To be more specific, the vanilla self-attention is finally defined by 
\begin{equation}
    \textbf{vanilla atten}: \vh_i^{(r)} = \sum_{j\in \mathcal{I}^{(r)}} \frac{\mathcal{G}(\vq_i^{(r)}, \vk_j^{(r)})}{\sum_{j\in \mathcal{I}^{(r)}} \mathcal{G}(\vq_i^{(r)}, \vk_j^{(r)})}\vv_j^{(r)}.
\end{equation}}

\subsection{Hierarchical attention}

In this section, we present HANO in Algorithm \ref{alg:Hierarchical Attention}, a hierarchically nested attention scheme with $\mathcal{O}(N)$ cost inspired by $\mathscr{H}^2$ matrices \cite{Hackbusch2015}, which is much more efficient than the vanilla attention above that scales with $\mathcal{O}(N^2)$. The overall HANO scheme (e.g., for a three-level example see Figure \ref{fig:vcycle}) resembles the V-cycle operations in multigrid methods, and it comprises four key operations: reduce, multilevel local attention, and decompose\&mix. In this procedure, instead of using global attention aggregation as in \eqref{eqn:vanilla_attn}, we utilize a local aggregation formula inspired by the $\mathcal{H}$ matrix approximation in the step of multilevel local attention. This approximation decomposes global interactions into local interactions at different scales ({ levels of tokens, denoted by $m$ of $\mathcal{I}^{(m)}$}). Empirically in Section \ref{sec:experiments}, this decomposition has a very minimal loss of expressivity.

\begin{figure}[htb]
    \centering
    \includegraphics[width=0.5\linewidth]{  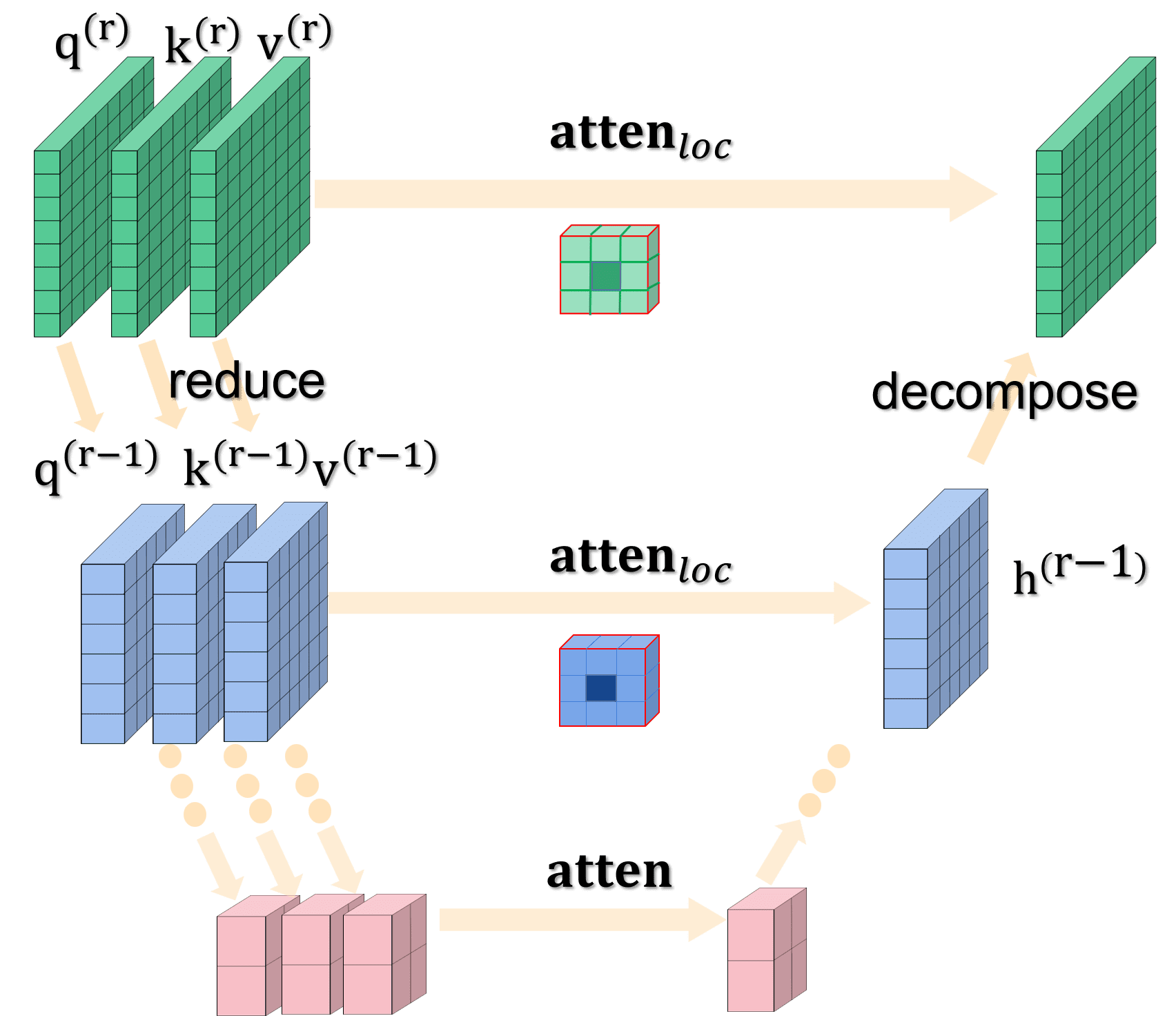}   
    \caption{\small Hierarchically nested attention.}
        \label{fig:vcycle} 
\end{figure}
\subsubsection{Reduce Operation \emph{using the quadtree hierarchy}}

The \emph{reduce} operation aggregates finer-level tokens into coarser-level tokens in the hierarchy. We denote $i^{(m, m+1)}$ as the set of indices of the $(m+1)$-th level child tokens of the $m$-th level token $i$, where $i\in \mathcal{I}^{(m)}$. 
In the quadtree case, $i^{(m,m+1)}=\{(i,0),(i,1),(i,2),(i,3)\}$, where $(i,j)$ is the concatenation of $i$ and $0\leq j\leq 3$. The \emph{reduce} map can be defined as 
$$
\vq_{i}^{(m)} = \mathcal{R}^{(m)}(\{\vq_{j}^{(m+1)} \}_{ j\in i^{(m,m+1)}}),
$$ 
which maps the $(m+1)$-th level tokens with indices in $i^{(m,m+1)}$ to the $m$-th level token $i$. We implement $\mathcal{R}^{(m)}$ as a linear layer, namely, 
$$\vq_i^{(m)} =\mR_0^{(m)}\vq_{(i,0)}^{(m+1) }+\mR_1^{(m)}\vq_{(i,1)}^{(m+1) }+\mR_2^{(m)}\vq_{(i,2)}^{(m+1) }+\mR_3^{(m)}\vq_{(i,3)}^{(m+1) },$$
where $\mR_{0}^{(m)}, \mR_{1}^{(m)}, \mR_{2}^{(m)}, \mR_{3}^{(m)} \in \mathbb{R}^{\mathcal{C}^{(m-1)}\times \mathcal{C}^{(m)}}$ are  matrices. The reduce operation is applied to $\boldsymbol{q}_{i}^{(m)}$, $\boldsymbol{k}_{i}^{(m)}$, and $\boldsymbol{v}_{i}^{(m)}$ for any $i\in \mathcal{I}^{(m)}$, and $m=r-1,\cdots,1$. This step corresponds to the downwards arrow in Figure \ref{fig:vcycle}. 


\subsubsection{Multilevel Local Attention}
Instead of using global attention aggregation as in \eqref{eqn:vanilla_attn}, we utilize a local aggregation formula in each single level.  The local aggregation at the $m$-th level $\textbf{atten}^{(m)}_{\textrm{loc}}$ is written using the nested $\boldsymbol{q}_i^{(m)}, \boldsymbol{k}_j^{(m)}, \boldsymbol{v}_j^{(m)}$ for $m=r,\ldots,1$ as follows: for $i \in \mathcal{I}^{(m)},$
\begin{equation}
 \textbf{atten}^{(m)}_{\textrm{loc}}: {\vh}_i^{(m)   } =  \sum_{j\in \mathcal{N}^{(m)}(i) \cup i} \mathcal{G}(\boldsymbol{q}_i^{(m) },     \boldsymbol{k}_{j}^{(m) }) \boldsymbol{v}_j^{(m) },
    \label{eqn:atten_loc_m}
\end{equation}
where $\mathcal{N}^{(m)}(i)$ denotes the set of neighbors of $i \in \mathcal{I}^{(m)}$ in the $m$-th level. { We define $\mathcal{N}^{(m)}(i)$ as the set of tokens within a specific window centered on the $i$-th token with a fixed window size for each level. This configuration ensures that attention aggregation mirrors the localized scope characteristic of convolution operations.}


\subsubsection{  Decompose\emph{\&}mix  Operation \emph{using the quadtree hierarchy}}

The \emph{decompose} operation reverses the reduce operation from level $1$ to level $r-1$. The decompose operator $\mathcal{D}^{(m)}: \vh_{i}^{(m)} \mapsto \{ \tilde{\vh}_{j}^{(m+1)}\}_{ j\in i^{(m,m+1)}}$, maps the $m$-th level feature $ {\vh}_{i}^{(m) }$ with index $i$ and $1\leq m\leq r-1$ to $(m+1)$-th level tokens associated to its child set $ i^{(m,m+1)}$. The presentation above provides an equivalent matrix form of $\mathcal{R}^{(m)}$ and $\mathcal{D}^{(m)}$ from fine to coarse levels.
 $ \tilde{\vh}_{i}^{(m+1)}$ is further aggregated to $ { \vh}_{i}^{(m+1)  }$ in the \emph{mix} operation such that $  \vh_{i}^{(m+1)  } +=   \tilde{\vh}_{i}^{(m+1)}$ for $i\in \mathcal{I}^{(m+1)}$. In the current implementation, we use a simple linear layer such that $\tilde{\vh}_{(i,s)}^{(m+1)}=\mD_s^{(m),T}  \vh_i^{(m)  }$, for $s=0,1,2,3$, with parameter matrices  $\mD_s^{(m)}\in \mathbb{R}  ^{\mathcal{C}^{(m)}\times \mathcal{C}^{(m+1)}}$. 


At this point, we can summarize the hierarchically nested attention algorithm as follows.
\begin{algorithm}[H]
    \caption{Hierarchically Nested Attention}
    \label{alg:Hierarchical Attention}
    \begin{algorithmic}   
    \STATE{\textbf{Input}: $\mathcal{I}^{(r)}$, $\vf_i^{(r) }$ for $i\in \mathcal{I}^{(r)}$.}\\  
    \textbf{STEP 0:} Compute $\boldsymbol{q}_i^{(r) }$, $\boldsymbol{k}_{i}^{(r) }$, $\boldsymbol{v}_i^{(r) }$ for $i\in \mathcal{I}^{(r)}$. \\
    \textbf{STEP 1: For} $m=r-1,\cdots,1$, \textbf{Do} the reduce operations $\boldsymbol{q}_{i}^{(m)  } = \mathcal{R}^{(m)}(\{\boldsymbol{q}_{j}^{(m+1)  }\}_{ j\in i^{(m,m+1)}} )$ and also for $\boldsymbol{k}_{i}^{(m)  }$ and $\boldsymbol{v}_{i}^{(m)  }$, for any $i\in \mathcal{I}^{(m)}$. \\
    \textbf{STEP 2: For} $m=r,\cdots,1$, \textbf{Do} the local aggregation by \eqref{eqn:atten_loc_m} to compute ${\vh}_{i}^{(m)  }, m=1,...,r,$  for any $i\in \mathcal{I}^{(m)}$.\\
    \textbf{STEP 3: For} $m=1,\cdots,r-1$, \textbf{Do} the decompose operations $\{\tilde{\vh}_{j}^{(m+1)}\}_{ j\in i^{(m,m+1)}} = \mathcal{D}^{(m)}( \vh_{i}^{(m)   })$, for any $i \in \mathcal{I}^{(m)}$; then $ \vh_{i}^{(m+1)   }+=\tilde{\vh}_{i}^{(m+1)}$,  for any $i \in \mathcal{I}^{(m+1)}$.\\
    \textbf{Output}: $\vh_i^{(r)  }$ for any $i\in \mathcal{I}^{(r)}$.
    \end{algorithmic}
\end{algorithm}


\subsubsection{Hierarchical Matrix Perspective}
\label{sec:Hmatrix}

The hierarchically nested attention in Algorithm \ref{alg:Hierarchical Attention} resembles the celebrated hierarchical matrix method \cite{Hackbusch2015}, in particular, the $\mathscr{H}^2$ matrix from the perspective of matrix operations. In the following, we take the one-dimensional binary tree-like hierarchical discretization shown in Figure \ref{fig:binary tree} as an example to illustrate the reduce operation, decompose operation, and multilevel token aggregation in \textbf{STEP 0-4} of Algorithm \ref{alg:Hierarchical Attention} using matrix representations.

\paragraph{\textbf{STEP 0}}
Given the input features $\vf^{(r) }$, compute the queries $\boldsymbol{q}_i^{(r) }$, keys $\boldsymbol{k}_{j}^{(r) }$, and values $\boldsymbol{v}_i^{(r) }$ for $j\in \mathcal{I}^{(r)}$.

Starting from the finest level features $\vf_i^{(r) },i \in \mathcal{I}^{(r)}$, the queries
$\boldsymbol{q}^{(r) }$ can be obtained by 
$$
\left[\begin{array}{c}
     \vdots  \\
     \boldsymbol{q}_i^{(r) }\\
     \vdots
\end{array}\right] = 
\underbrace{\left[
\begin{array}{llll}
\mW^{Q,(r)} & & \\
& \mW^{Q,(r)} & & \\
& & \ddots & \\
& & & \mW^{Q,(r)}
\end{array}
\right]}_{|\mathcal{I}^{(r)}|}
\left.\left[\begin{array}{c}
     \vdots  \\
     \vf_i^{(r) }\\
     \vdots
\end{array}\right]\right\}|\mathcal{I}^{(r)}|,
$$
and for the keys $\boldsymbol{k}^{(r) }$ and values $\boldsymbol{v}^{(r)}$, similar procedures follow. 

\paragraph{\textbf{STEP 1}}
\textbf{For} $m=r-1:1$, \textbf{Do} the reduce operations $\boldsymbol{q}_{i}^{(m) } = \mathcal{R}^{(m)}(\{\boldsymbol{q}_{j}^{(m+1) }\}_{ j\in i^{(m,m+1)}} )$ and also for $\boldsymbol{k}_{i}^{(m) }$ and $\boldsymbol{v}_{i}^{(m) }$, for any $i\in \mathcal{I}^{(m)}$. \\
If $\mathcal{R}^{(m)}$ is linear, the reduce operations correspond to $
\left[\begin{array}{c}
     \vdots  \\
     \boldsymbol{q}_i^{(m) }\\
     \vdots
\end{array}\right] = {\tR^{(m) }} \left[\begin{array}{c}
     \vdots  \\
     \boldsymbol{q}_i^{(m+1) }\\
     \vdots
\end{array}\right] 
$.
In matrix form, the reduce operation is given by multiplying with
$$
{\tR^{(m)}}:=\underbrace{
\left[\begin{array}{llllllll }
\mR_{0}^{(m)} & \mR_{1}^{(m)}  &  &   &  &   &  &     \\
 &   &    \mR_{0}^{(m)} & \mR_{1}^{(m)}   \\
  &  &  &  &\ddots &\ddots &\\
    &    &   &  &  &  &\mR_{0}^{(m)} & \mR_{1}^{(m)} 
\end{array}\right]
}_{|\mathcal{I}^{(m+1)}|} 
\left.\vphantom{\begin{array}{c}
       \cr
       \\[1.2ex]
       \cr
       \\[1.2ex]
\end{array}} \right\} |\mathcal{I}^{(m)}|,
$$
and $\mR_{0}^{(m)}, \mR_{1}^{(m)}\in \mathbb{R}^{\mathcal{C}{(m-1)}\times \mathcal{C}{(m)}}$ are matrices parametrized by linear layers. In practice, queries, keys, and values use different $\mR_{0}^{(m)}, \mR_{1}^{(m)}$ to enhance the expressivity. In general, these operators $\mathcal{R}^{(m)}$ are not limited to linear operators. The composition of nonlinear activation functions would help increase the expressivity. The nested learnable operators $\mathcal{R}^{(m)}$ also induce the channel mixing and is equivalent to a structured parameterization of $\mW^{Q}, \mW^{V}, \mW^{K}$ matrices for the coarse level tokens, in the sense that, inductively,

\begin{equation}
\begin{aligned}
\left[\begin{array}{c}
     \vdots  \\
     \boldsymbol{q}_i^{(m) }\\
     \vdots
\end{array}\right] & ={\tR^{(m)}} \cdots {\tR^{(r-1)}} \left[\begin{array}{c}
     \vdots  \\
     \boldsymbol{q}_i^{(r) }\\
     \vdots
\end{array}\right]
\\
& = {\tR^{(m)}} \cdots {\tR^{(r-1)}} 
\left[
\begin{array}{llll}
\mW^{Q,(r)} & & \\
& \mW^{Q,(r)} & & \\
& & \ddots & \\
& & & \mW^{Q,(r)}
\end{array}
\right]
 \left[\begin{array}{c}
     \vdots  \\
     \boldsymbol{f}_i^{(r) }\\
     \vdots
\end{array}\right].
\end{aligned}
\end{equation}

\begin{figure}[htb]
    \centering
    \includegraphics[width=0.8\textwidth]{  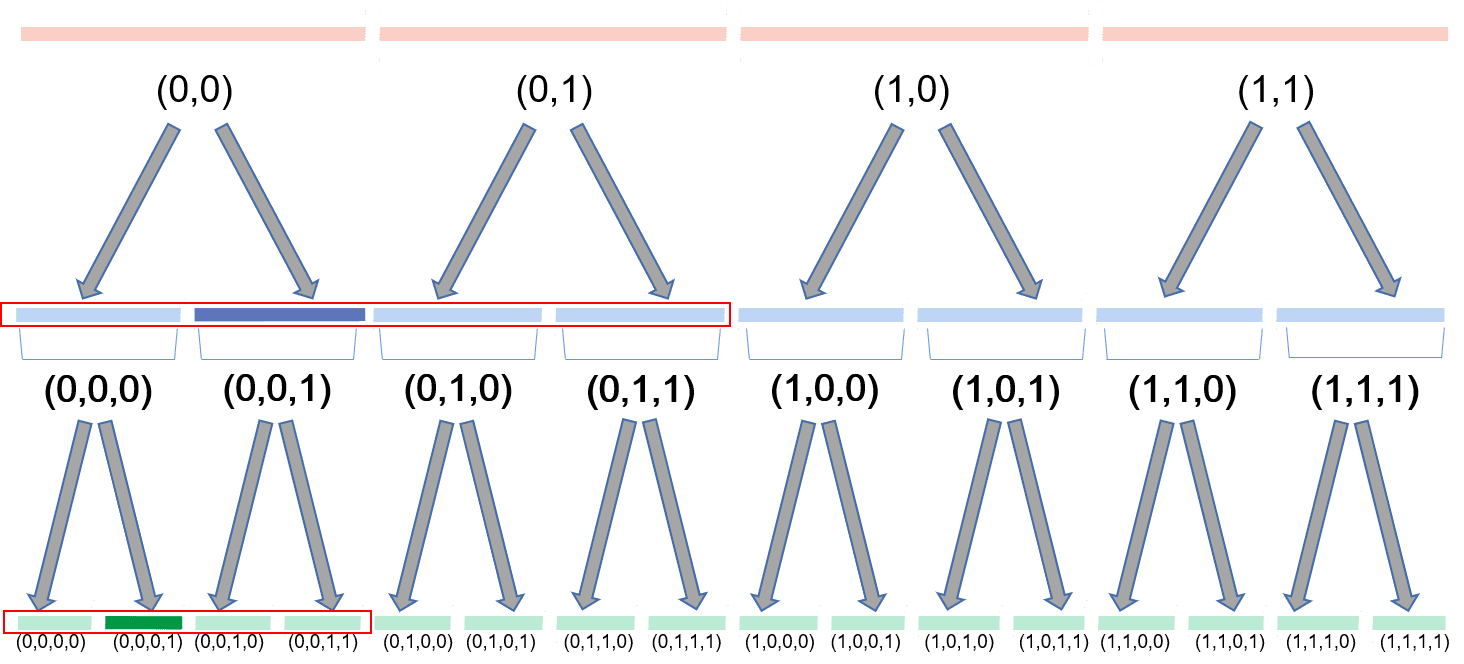}
    \caption{Hierarchical discretization of 1D domain.  The coarsest level partition is plotted as the top four segments in pink. The segment $(0,0)$ is further partitioned into two child segments $(0,0,0)$ and $(0,0, 1)$. During the reducing process, the computation proceeds from bottom to top to obtain coarser-level representations. For example, the $(0,0)$ representations are obtained by applying learnable reduce operations $\mathcal{R}^{(2)}$ and $\mathcal{R}^{(1)}$ on $(0,0,0)$ and $(0,0,1)$ respectively. When generating the high-resolution representations, the computation proceeds from top to bottom by applying learnable decomposition operations $\mathcal{D}^{(1)}$ and $\mathcal{D}^{(2)}.$ The red frames show examples of attention windows at each level.}
    \label{fig:binary tree}
\end{figure}

\paragraph{\textbf{STEP 2}}
With the $m$-th level queries and keys, we can calculate the local attention matrix $\mG_{\rm loc}^{(m)}$ at the $m$-th level with 
$(\mG_{\rm loc}^{(m)})_{i,j}:=\exp(\boldsymbol{q}_i^{(m) } \cdot \boldsymbol{k}_{j}^{(m) })$ for $i\in \mathcal{N}^{(m)}(j)$, or $i\sim j$. 

    
\paragraph{\textbf{STEP 3}}
The decompose operations, opposite to the reduce operations, correspond to the transpose of the following matrix in the linear case,
$$
{\tD^{(m)}}:=\underbrace{
\left[\begin{array}{llllllll }
\mD_{0}^{(m)} & \mD_{1}^{(m)} &   &   &  &   &  &     \\
 &   &    \mD_{0}^{(m)} & \mD_{1}^{(m)}   \\
  &  &  &    &\ddots & \ddots& \\
   & & &  &  &  &\mD_{0}^{(m)} & \mD_{1}^{(m)}  
\end{array}\right]
}_{|\mathcal{I}^{(m+1)}|} 
\left.\vphantom{\begin{array}{c}
       \cr
       \\[1.2ex]
       \cr
       \\[1.2ex]
\end{array}} \right\} |\mathcal{I}^{(m)}|,
$$
The $m$-th level aggregation in Figure \ref{fig:vcycle} contributes to the final output ${\vf}^{(r)}$ in the form 
$$
\tD^{(r-1),\tT} \cdots \tD^{(m), \tT} \mG_{\rm loc}^{(m)} \tR^{(m)} \cdots \tR^{(r-1)} \left[\begin{array}{c}
     \vdots  \\
     \boldsymbol{v}_i^{(r) }\\
     \vdots
\end{array}\right].
$$
Eventually, aggregations at all $r$ levels in one V-cycle can be summed up as
\begin{equation}
    \left[\begin{array}{c}
     \vdots  \\
     \vh_i^{(r)   }\\
     \vdots
\end{array}\right] = \left( \sum_{m=1}^{r-1} (\tD^{(r-1),\tT} \cdots \tD^{(m), \tT} \mG_{\rm loc}^{(m)} \tR^{(m)} \cdots \tR^{(r-1)})  + \mG_{\rm loc}^{(r)} \right)\left[\begin{array}{c} 
     \vdots  \\
     \boldsymbol{v}_i^{(r) }\\
     \vdots
\end{array}\right].
\label{eqn: 3matrix}
\end{equation}
The hierarchical attention matrix  
$$
\mG_h:= \sum_{m=1}^{r-1} (\tD^{(r-1),\tT} \cdots \tD^{(m), \tT} \mG_{\rm loc}^{(m)} \tR^{(m)} \cdots \tR^{(r-1)})  + \mG_{\rm loc}^{(r)},
$$ 
in \eqref{eqn: 3matrix} resembles the three-level $\mathscr{H}^2$ matrix decomposition illustrated in Figure \ref{fig:Hmatrix} (see also \cite{Hackbusch2015} for a detailed description). The sparsity of $\mG_h$ lies in the fact that the attention matrix is only computed for pairs of tokens within the neighbor set. The $\mathscr{H}^2$ matrix-vector multiplication in \eqref{eqn: 3matrix} implies the $\mathcal{O}(N)$ complexity of Algorithm \ref{alg:Hierarchical Attention}.

\begin{figure}[htbp]
    \centering
    \includegraphics[width=0.8\textwidth]{  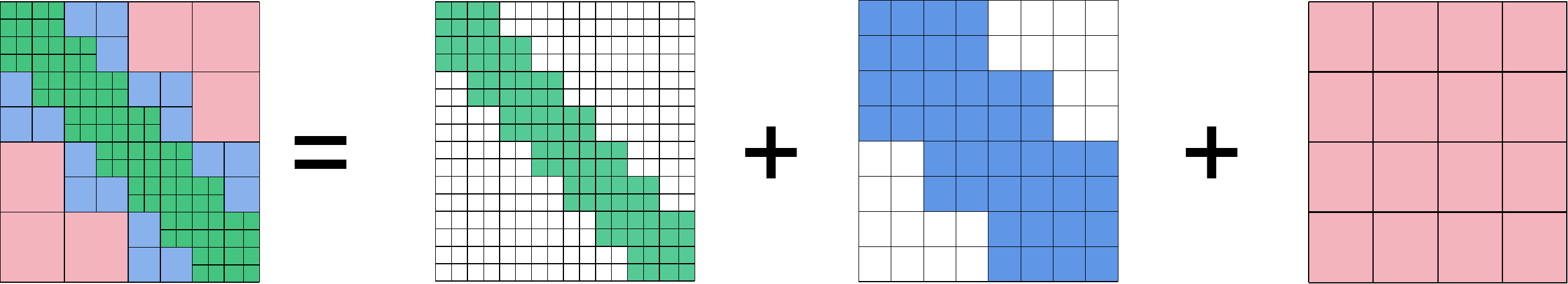}
    \caption{A demonstration of the decomposition of attention matrix into three levels of local attention matrix.}
    \label{fig:Hmatrix}
\end{figure}

Note that, the local attention matrix at level $\mathcal{I}^{(1)}$ (pink), level $\mathcal{I}^{(2)}$ (blue) and level $\mathcal{I}^{(3)}$ (green) are   $\mG_{loc}^{(1)}$, $\mG_{loc}^{(2)}$  and $\mG_{loc}^{(3)}$, respectively. However, when considering their contributions to the finest level, they are equivalent to the attention matrix $\tD^{(2),\tT} \tD^{(1), \tT} \mG_{loc}^{(1)} \tR^{(1)}  \tR^{(2)} \in \mathbb{R}^{\mathcal{I}^{(3)}} \times \mathbb{R}^{\mathcal{I}^{(3)}}$ (pink), $\tD^{(2),\tT}   \mG_{loc}^{(2)}   \tR^{(2)} \in \mathbb{R}^{\mathcal{I}^{(3)}} \times \mathbb{R}^{\mathcal{I}^{(3)}}$ (blue) and $\mG_{loc}^{(3)}$ (green), as demonstrated in Figure \ref{fig:Hmatrix}. Each pink block and blue block are actually low-rank sub-matrices with rank $\mathcal{C}^{(1)}$ and rank $\mathcal{C}^{(2)}$, respectively, by definition.

\subsubsection{Complexity}
\label{sec:compleixty}

We conclude the section by estimating the complexity of Algorithm \ref{alg:Hierarchical Attention}.
\begin{proposition} The reduce operation, multilevel local attention, and decomposition/mix operations together form a V-cycle for updating tokens, as illustrated in Figure \ref{fig:vcycle}. The cost of one V-cycle is $O(N)$ if $\mathcal{I}$ is a quadtree, as implemented in the paper. 
\end{proposition}

\begin{proof}
For each level $m$, the cost to compute \eqref{eqn:atten_loc_m} is $c(|\mathcal{I}^{(m)}|\mathcal{C}^{(m)})$ since for each $i \in \mathcal{I}^{(m)}$ the cardinality of the neighbour set $\mathcal{N}^{(m)}(i) $ is bounded by a constant $c$. The reduce operation $\vf_{i}^{(k-1)} = \mathcal{R}^{(k-1)}(\{\vf_{j}^{(m)}\}_{ j\in i^{(k-1,k)}} )$ costs at most $|\mathcal{I}^{(m)}|\mathcal{C}^{(m)}\mathcal{C}^{(k-1)}$ flops and so does the decompose operation at the same level. Therefore, for each level, the operation cost is $c(|\mathcal{I}^{(m)}|\mathcal{C}^{(m)})+ 2|\mathcal{I}^{(m)}|\mathcal{C}^{(m)}\mathcal{C}^{(m-1)}$. When $\mathcal{I}$ is a quadtree, $\mathcal{I}^{(r)} = N, \, \mathcal{I}^{(r-1)} = N/4, \cdots, \mathcal{I}^{(1)} = 4$, therefore the total computational cost $\sim \mathcal{O}(N)$.  
\end{proof}


\subsection{Overall Architecture}
The overall neural network architecture uses the standard Transformer \cite{dosovitskiy2020image} architecture for computer vision tasks, and the HANO attention is a drop-in replacement of the attention mechanism therein. The input $\va$ is first embedded into $n \times n$ tokens represented as a tensor of size $n \times n \times \mathcal{C}^{(r)}$ using patch embedding, for a dataset with resolution $n_{f}\times n_f$, such as in the multiscale elliptic equation benchmark. These tokens are then processed by a multi-level hierarchically nested attention, as described in Section \ref{sec:methods}, resulting in hidden features ${\vh_i^{(r)}, i \in \mathcal{I}^{(r)}}$. Finally, a decoder maps the hidden features to the solution $\vu$. Different decoders can be employed depending on prior knowledge of the PDE model. For example, \cite{lu2021learning} uses a simple feedforward neural network (FFN) to learn a ``basis'' set, \cite{bhattacharya2021model} employs a data-driven SVD-based decoder, and in our work, the compose and mix operations function as the decoder. 

In this paper, we choose $r=5$ as the depth of the HANO, window size of $3\times 3$ for the definition of the neighborhood $\mathcal{N}^{(\cdot)}(\cdot)$ in \eqref{eqn:atten_loc_m}, GELU as the activation function, and a CNN-based patch embedding module to transfer the input data into features/tokens.

For a dataset with resolution $n_{f}\times n_f$, such as in the multiscale elliptic equation benchmark \ref{sec:experiments:darcy}, the input feature $\vf^{(5)}$ is represented as a tensor of size $n \times n \times C$ via patch embedding. The self-attention is first computed within a local window on level $5$. Then the reduce layer concatenates the features of each group of $2 \times 2$ neighboring tokens and applies a linear transformation on the $4C$-dimensional concatenated features on $\frac{n}{2} \times \frac{n}{2}$ level 2 tokens, to obtain level 2 features $\vf^{(2)}$ as a tensor of the size $\frac{n}{2} \times \frac{n}{2} \times 2C$. The procedure is repeated from level 2 to level 1 with $\vf^{(1)}$ of size $\frac{n}{4} \times \frac{n}{4} \times 4C$. 

For the decompose operations, starting at level 1, a linear layer is applied to transform the $4C$-dimensional features $\vf^{(1)}$ into $8C$-dimensional features.
Each level 1 token with $8C$-dimensional features is decomposed into four level 2 tokens with $2C$-dimensional features.
These four level 2 tokens are added to the existing level 2 feature $\vf^{(2)}$ with output size of $\frac{n}{2} \times \frac{n}{2} \times 2C$.
The decomposition procedure is repeated from level 2 to level 3.
The output of level 3 is $\vf^{(3)}$, which has a size of $n \times n \times C$. We call the above procedures a cycle and we repeat $k$ cycles by the same set up with layer normalizations between cycles. 

The detailed configuration for HANO, which may consists of different levels and feature dimensions for each task, is presented in Table \ref{tab:Hyperparameters configuration}.


\begin{table}[H]
    \centering
    \begin{tabular}{cc}
    \hline \textsc{Module} & \textsc{Hyperparameters}   \\
    \hline \textsc{Patch embedding}   & \textsc{patch size}: 4, \textsc{padding}: 0 \\
    \hline \textsc{Hierarchical Attention} & $\begin{array}{c}
      \textsc{number of levels: }  5     \\
      \textsc{down sampling ratio} {|\mathcal{I}^{m+1}|}/{|\mathcal{I}^{m}|}: 4 \\
        \textsc{feature dimension at each level: }   
 \{32, 32, 32, 32, 32\} \\
        \textsc{window size at each level: } \{3, 3, 3, 3, 3\}\\
        \textsc{LayerNorm position: after attention} \\
        \textsc{number of cycles: 2}
    \end{array}$   \\
    \hline    
    \end{tabular}
    \caption{Hyperparameters configurations}
    \label{tab:Hyperparameters configuration}
\end{table}

\subsection{Comparison with Existing Multilevel Transformers}
 In vision transformers like \cite{liu2021swin, zhang2021aggregating} with a multilevel architecture, attentions are performed at each level separately, resulting in no multilevel attention-based aggregation. This may lead to the loss of fine-scale information in the coarsening process, which is not ideal for learning multiscale operators where fine-scale features are crucial. The following components in the HANO approach we proposed could potentially address this issue: 
 \begin{enumerate}[label=(\arabic*)]
     \item Attention-based local aggregations at each level, followed by summation of features from all levels to form the updated fine-scale features;
     \item The reduce/local aggregation/decompose/mix operations, inspired by the $\mathcal{H}^2$ hierarchical matrix method, enable the recovery of fine details with a linear cost; 
     \item Nested computation of features at all levels, with simultaneous parameterization of the learnable matrices $\mW^{Q}, \mW^{V}, \mW^{K}$ in a nested manner. Those components highlight the novelty of our method. 
 \end{enumerate}
 
Meanwhile, the nested token calculation approach also differs from existing multilevel vision transformers \cite{liu2021swin, zhang2021aggregating}, as we perform the reduce operation before attention aggregation, resulting in nested $\vq, \vk, \vv$ tokens. Additionally, our approach differs from UNet \cite{ronneberger2015u}, which utilizes a maxpooling for the reduce operation. For a numerical ablation study in which UNet and SWIN have the same general architecture, but different ways to aggregate features in each level, please refer to Table \ref{tab:bench}. 

Our attention matrix has a global interaction range but features low-rank off-diagonal blocks at each level, as shown in Section \ref{sec:Hmatrix}. Note that the overall attention matrix itself is not necessarily low-rank, distinguishing it from efficient attention models using kernel tricks or low-rank projections~\cite{choromanski2020rethinking,linformer:2020,peng2021random,nguyen2021fmmformer, xiong2021nystromformer}.

\section{Experiments}
\label{sec:experiments}

In this section, we tested HANO's evaluation accuracy and efficiency compared to other state-of-the-art neural operators and other Transformers in several standard operator learning benchmarks. In Section \ref{sec:experiments:darcy}, a new operator learning benchmark is created for solving multiscale elliptic PDEs, and common neural operators are tested. HANO demonstrates higher accuracy and robustness for coefficients with different degrees of roughness/multiscale features. 
In Section \ref{sec:experiments:navierstokes}, HANO is tested in the Navier-Stokes equation benchmark problem with a high Reynolds number. In Section \ref{sec:experiments:helmholtz}, HANO is tested in a benchmark for the Helmholtz equation.

\subsection{Data generation for porous media benchmark}
\label{sec:experiments:darcy}
We apply the HANO model to learn the mapping from coefficient functions to solution operators for multiscale elliptic equations. We use the porous media benchmark with a two-phase coefficient produced from a log-normal random field following e.g. \cite{guadagnini1999nonlocal,gittelson2010stochastic}, and popularized by \cite{nelsen2021random,li2020fourier} as a standard task in operator learning. The model equation in divergence form writes
\begin{equation}
    \left\{
	\begin{aligned}
		-\nabla \cdot(a   \nabla u  ) &=f  & &\text{ in }  D \\
		u  &=0 &  &\text{ on } \partial D
	\end{aligned}
    \right.
\label{eqn:darcy}
\end{equation}
where the coefficient $0<a_{\min}\leq a:=a(x) \leq a_{\max}, \forall x \in D$, and the forcing term $f\in L^2(D;\mathbb{R})$. By the Lax-Milgram lemma, the coefficient to solution map $\mathcal{S}:L^\infty (D;\mathbb{R}_+) \rightarrow H_0^1(D;\mathbb{R})$, $u\mapsto \mathcal{S}(a)$ is well-defined.

 We also include experiments for multiscale trigonometric coefficients with higher contrast. The newly generated benchmark using rough or multiscale coefficient $a(x)$ test the capacity of neural operators to capture fast oscillation (e.g. the diffusion coefficient becomes $a_{\epsilon}(x) = a(x/\varepsilon)$ with $\varepsilon\ll 1$), higher contrast ratio $a_{\max}/a_{\min}$, and even a continuum of non-separable scales. See Section \ref{sec:additonaltests:multiscaletrignometric} for details. Results for three benchmarks are summarized in Table \ref{tab:bench}. 

\subsubsection{Two-Phase Coefficient}
\label{sec:datageneration}

The two-phase coefficients $\{a\}$ and approximations to solutions $\{u\}$ in Section \ref{sec:experiments:darcy} are generated according to \url{https://github.com/zongyi-li/fourier_neural_operator/tree/master/data_generation} as a standard benchmark.  The forcing term is fixed as $f(x)\equiv 1$ in $D$. The coefficients $a(x)$ are generated according to $a \sim \mu:=\psi_{\#} \mathcal{N}\left(0,(-\Delta+c I)^{-2}\right)$, where the covariance is inverting this elliptic operator with zero Neumann boundary conditions. The mapping $\psi: \mathbb{R} \rightarrow \mathbb{R}$ takes the value $a_{\max}$ on the positive part of the real line and $a_{\min}$ on the negative part. 
The push-forward is defined in a pointwise manner, thus $a$ takes the two values inside $D$ and jumps from one value to the other randomly with likelihood characterized by the covariance. Consequently, $a_{\max}$ and $a_{\min}$ can control the contrast of the coefficient. The parameter $c$ controls the ``roughness'' of the coefficient; a larger $c$ results in a coefficient with rougher two-phase interfaces, as shown in Figure \ref{fig:smooth&rough}. Solutions $u$ are obtained by using a second-order finite difference scheme on a staggered grid with respect to $a$.

In \cite{li2020fourier} and all subsequent work benchmarking this problem for operator learning, the coefficient is determined using $a_{\max} = 12$, $a_{\min} = 3$, and $c=9$, which results in a relative simply topology of the interface. In this case, the solutions are also relatively smooth (which is referred to as ``\emph{Darcy smooth}''). To show the architectural advantage of HANO, we adjust the parameters to increase the likelihood of the random jumps of the coefficients, which results in much more complicated topology of the interfaces, and solutions generated show more ``roughness'' (which is referred to as ``\emph{Darcy rough}''). See Figure \ref{fig:smooth&rough} for an example.

\subsubsection{Multiscale trigonometric coefficient}
\label{sec:additonaltests:multiscaletrignometric}

We also consider \eqref{eqn:darcy} with multiscale trigonometric coefficient adapted from \cite{OwhadiMultigrid:2017}, as one of the multiscale elliptic equation benchmarks. The domain $D$ is $(-1,1)^2$, and the coefficient $a(x)$ is defined as 
$$
a(x) = \prod \limits_{k=1}^6  (1+\frac{1}{2} \cos(a_k \pi (x_1+x_2)))(1+\frac{1}{2} \sin(a_k \pi (x_2-3x_1))),
$$ 
where $a_k$ is uniformly distributed between $2^{k-1}$ and $1.5\times 2^{k-1}$ for each $k$, and the forcing term is fixed as $f(x)\equiv 1$. The reference solutions are obtained using the linear Lagrange finite element methods on uniform triangulation cut from a $1023 \times 1023$ Cartesian grid. Datasets of lower resolution are created by downsampling the higher resolution dataset using bilinear interpolation. The experiment results for the multiscale trigonometric case with different resolutions are shown in Table  \ref{tab:bench}. HANO obtains the best relative $L^2$ error compared to other neural operators. See Figures \ref{fig:gamblet_fig} and \ref{fig:gamblet_fig2} for illustrations of the coefficient and comparison of the solutions/derivatives at the slice $x=0$.

\subsection{Training Setup}
\label{sec:experiments:setup}

\def\ba{\boldsymbol{a}}
\def\bu{\boldsymbol{u}}

We consider pairs of functions $\{(a_j,u_j)\}_{j=1}^N$, where $a_j$ is drawn from a probability measure specified in Sections \ref{sec:datageneration} and \ref{sec:additonaltests:multiscaletrignometric}, and $u_j = \mathcal{S}(a_j)$. During training and evaluation, $a_j$ and $u_j$ are evaluated pointwisely on a uniform 2D grid $\mathsf{G}^2:=\{(x_1,x_2)=(ih,jh) \mid i,j=0,\dots,n-1\}$ as matrices $\va_j$ and $\vu_j$. We generate the hierarchical index tree $\mathcal{I}$ using a quadtree representation of nodes with depth $r$, where the finest level objects are pixels or patches aggregated by pixels. 

We apply the ADAM optimizer with a maximum learning rate $10^{-3}$, weight decay $10^{-4}$, and a 1-cycle scheduler from \cite{smith2019super}. We choose batch size 8 for experiments in Sections \ref{sec:experiments:darcy} and \ref{sec:experiments:navierstokes}, and batch size 4 for experiments in Sections \ref{sec:additonaltests:multiscaletrignometric} and \ref{sec:experiments:helmholtz}.  

For the Darcy rough case, we use a train-validation-test split of 1280, 112, and 112, respectively, with a max of 500 epochs. For Darcy smooth and multiscale trigonometric cases, we use a split of 1000, 100, and 100, respectively, with a max of 500 epochs for the Darcy smooth case and 300 for the multiscale trigonometric case. 

Baseline models are taken from the publicly available official implementations, and changes are detailed in each subsection if there is any. All experiments are run on an NVIDIA A100 GPU.

\paragraph{Empirical $H^1$ Loss Function}
For multiscale problems, 
we adopt an $H^1$ loss function instead of the conventional $L^2$ loss, which places greater emphasis on high-frequency components. Empirically, we observe that the model's training is more efficient and the generalization is more robust than those without. First, the empirical $L^2$ loss function is defined as 
$$
\mathcal{L}^{L}(\left\{(\ba_j, \bu_j)\right\}_{j=1}^N; \theta):=  \frac{1}{N}\sum_{i=1}^N \|\bu_j-\mathcal{N}(\ba_j; \theta)\|_{l^2}/\|\bu_j\|_{l^2},
$$ where $\|\cdot\|_{l^2}$ is the canonical $l^2$ vector norm. The normalized discrete Fourier transform (DFT) coefficients of $f$ are given by 
\begin{equation}
\label{eq:dft}
    \mathcal{F}(f)(\xi):= \frac{1}{\sqrt{n}}\sum_{x\in \mathsf{G}^2}  f(x) e^{-2 i \pi  x \cdot \xi },
\quad \xi \in \mathbb{Z}_n^{2}:=\left\{\xi = (\xi_1, \xi_2) \in \mathbb{Z}^2 \mid  -n/2+1\leqslant \xi_{j} \leqslant n/2, j=1, 2\right\}
\end{equation}
The empirical $H^1$ loss function is thus given by,
\begin{equation}
\label{eq:loss-h1-frequency}
        \mathcal{L}^{H}(\left\{(\ba_j, \bu_j)\right\}_{j=1}^N; \theta):=\frac{1}{N}\sum_{i} \|\vu_j-\mathcal{N}(\ba_j; \theta)\|_{h}/\|\bu_j\|_{h},  
\end{equation}
where $\|\vu\|_{h} := \sqrt{\sum_{\xi\in \mathbb{Z}_n^{2}} |\xi|^2(\mathcal{F}(\vu)(\xi))^2}$. $\mathcal{L}^{H}$ can be viewed as a weighted $\mathcal{L}^{L}$ loss using $|\xi|^2$ weights to balance the error in low- and high-frequency components. Note that the frequency domain representation of the discrete $H^1$ norm is used following the practices in e.g. \cite{ronen2019convergence,Matthew2020Fourier}. Here the discrete $H^1$ norm approximated using difference quotient in the physical space can also be employed, however, from the numerical quadrature point of view, by Parseval identity \eqref{eq:loss-h1-frequency} is exact with no quadrature error.





\begin{figure}[htbp]
    \centering
    \subfigure[\scriptsize coefficient]{\includegraphics[width=0.18\textwidth,height=0.25\textwidth]{  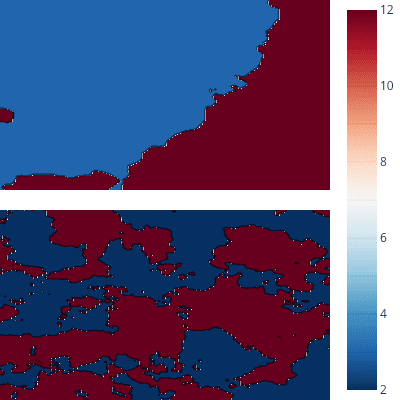}}
    \subfigure[\scriptsize reference]{\includegraphics[width=0.18\textwidth, height=0.25\textwidth]{  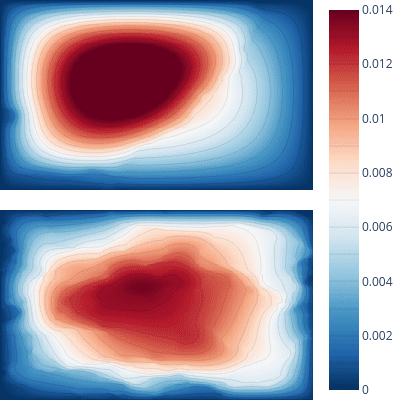}}
    \subfigure[\scriptsize HANO prediction]{\includegraphics[width=0.18\textwidth, height=0.25\textwidth]{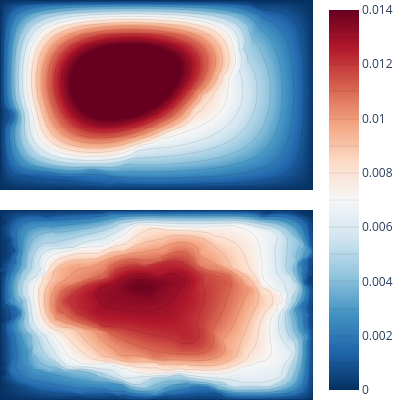}}
    \subfigure[\scriptsize abs. error of HANO]
    {\includegraphics[width=0.18\textwidth, height=0.25\textwidth]{    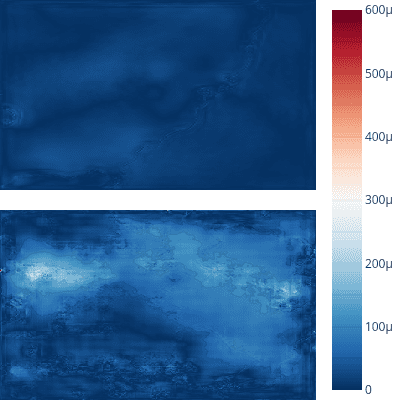}}
    \subfigure[\scriptsize abs. error of FNO2D]
    {\includegraphics[width=0.18\textwidth, height=0.25\textwidth]{  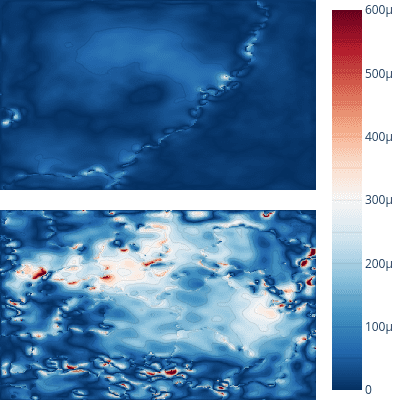}}

    \caption{  \textbf{Top:} (a) smooth coefficient in \cite{li2020fourier}, with $a_{\max}=12,$ $a_{\min}=3$, $c=9$, (b), reference solution, (c) HANO solution, (d) HANO, absolute (abs.) error, (e) FNO2D, abs. error; \textbf{Bottom:} (a) rough coefficients with $a_{\max}=12,$ $a_{\min}=2$, $c=20$, (b) reference solution, (c) HANO solution, (d) HANO, abs. error, (e) FNO2D, abs. error, the maximal error of FNO2D is around $900\mu=9\mathrm{e}{-4}$.  }
    \label{fig:smooth&rough}
\end{figure}

\begin{table*}[ht]
    \caption{The baseline methods are implemented with their official implementation if publicly available.  Performance are measured with relative $L^2$ errors ($\times 10^{-2}$) and relative $H^1$ errors ($\times 10^{-2}$). For the Darcy rough case, we run each experiment 3 times to calculate the mean and the standard deviation (after $\pm$) of relative $L^2$ and relative $H^1$ errors. All experiments use a fixed train-val-test split setup, see Section \ref{sec:experiments:setup} for details.  
}
    \label{tab:bench}
    \begin{center}
    \begin{tabular}{llrrrrrrr}
    \toprule
    && & \multicolumn{2}{c}{\textbf{Darcy smooth}} & \multicolumn{2}{c}{\textbf{Darcy rough}} & \multicolumn{2}{c}{\textbf{Multiscale}}\\ \cmidrule(lr){4-5}\cmidrule(lr){6-7}\cmidrule(lr){8-9}
    &\textbf{Model} & Runtime (s) &  $L^2$ & $H^1$ & $L^2$ & $H^1$ & $L^2$ & $H^1$ \\ 
    \midrule
    \multirow{4}{*}{\rotatebox[origin=c]{90}{ }} 
     &\textsc{FNO2D} 
    & $\textbf{7.278}$ 
    & $0.706$ & $3.131$
    & $1.782$ {\scriptsize $\pm0.021$}  & $9.318$ {\scriptsize $\pm0.088$}
    & $1.949$& $14.535$ \\
     &\textsc{FNO2D $H^1$} 
    & $7.391$ 
    & $0.684$ & $2.583$
    & $1.613$ {\scriptsize $\pm0.010$} & $7.516$ {\scriptsize $\pm0.049$}
    & $1.800$ & $9.619$  \\
    &\textsc{UNet} 
    &   $9.127$ 
    & $2.169$  & $4.885$ 
    &  $3.591$ {\scriptsize  $\pm0.127$} &  $6.479$ {\scriptsize $\pm0.311$}
    & $1.425$ & $5.012$  \\
    &\textsc{U-NO} 
    &   $11.259$
    & $0.678$ &   $2.580$
    &  $1.185$ {\scriptsize $\pm0.005 $} &  $5.695$ {\scriptsize $\pm0.005$}
    & $1.350$ &  $8.577$  \\
    &\textsc{U-NO $H^1$} 
    &   $11.428$
    & $0.492$ &   $1.276$
    &  $1.023$  {\scriptsize $\pm0.013$}  & $3.784$ {\scriptsize $\pm0.016$}
    & $1.187$ & $5.380$  \\
    &\textsc{MWT} & 
    $19.715$ & --- & --- 
    & $1.138$ {\scriptsize $\pm0.010$}& $4.107$  {\scriptsize $\pm0.008$}
    & $1.021$ & $7.245$ \\ 
    &\textsc{GT}
    & $38.219$ 
    & $0.945$ & $3.365$
    & $1.790$ {\scriptsize $\pm0.012$}& $6.269$ {\scriptsize $\pm0.418$}
    & $1.052$& $8.207$  \\ 
     &\textsc{SWIN} 
     & $41.417$ 
     &  --- & --- 
     & $1.622$ {\scriptsize $\pm0.047$}& $6.796$ {\scriptsize $\pm0.359$}
     & $1.489$ & $13.385$ \\
     &{  \textsc{HANO $L^2$ }} 
    &  $9.620$
    &  $0.490$  &  $1.311$  &   {$0.931$ {\scriptsize $\pm0.021$}} &  {$2.612$ {\scriptsize $\pm0.059$}}
    & $0.842$ & $4.842$\\
     &\textsc{HANO $H^1$} 
    &  $9.620$
    & \bm{$0.218$ } & \bm{$0.763$}  &   {\bm{$0.343$ }{\scriptsize $\pm0.006$}} &  {\bm{$1.846$}{\scriptsize $\pm0.023$}}
    & \bm{$0.580$} & \bm{$1.749$}\\ 
    
    \bottomrule\\[-2.5mm]
    \multicolumn{9}{l}{{\scriptsize --- MWT \cite{gupta2021multiwavelet} only supports resolution with powers of two.}}
    \end{tabular}
    \end{center}

\end{table*}

\subsection{Empirical Study on the Spectral Bias in Operator Learning}

We compare HANO and FNO in terms of prediction error dynamics across frequencies from epoch 0 to epoch 100 (end) in Figure \ref{fig:spectral_bias} for a comprehensive comparison. The subfigures (c,d) in Figure \ref{fig:spectral_bias} suggest that existing methods can learn low frequencies quickly but struggle with higher frequencies. At the end of the training, plenty of high-frequency components are still not well resolved as shown in Figure \ref{fig:gamblet_fig}(b,c,d) and Figure \ref{fig:gamblet_error_freqerror}. This phenomenon is often referred to as the \emph{spectral bias}, well-documented for training neural networks for conventional classification tasks, and here we observe it in operator learning tasks. On the contrary, HANO's error decays faster for higher frequencies and more uniformly overall. It also achieves lower testing errors. 
Experimentally, the ablation suggests that the hierarchical nested attention allows the model to capture finer-scale variations better. which also helps HANO outperform existing methods as is shown in Figure~\ref{fig:gamblet_fig}.
We also observe that, with the $H^1$ loss function, spectral bias is further mitigated, which applies to FNO-based variants as well.



To better illustrate the spectral bias of multiscale operator learning, we record the training dynamics in the frequency domain. Recall that, the spatial domain $D=[0,1]^2$ is discretized uniformly with $h=1/n$ to yield a Cartesian grid $\mathsf{G}^2$ in our experiments. 
For any $\xi \in \mathbb{Z}_n^{2}$, consider the normalized discrete Fourier transform (DFT) coefficients $\mathcal{F}(f)(\cdot)$ of $f$ in \eqref{eq:dft}, 
the mean absolute prediction error for a given frequency $\xi \in \mathbb{Z}_n^{2}$ is measured by 
$$
\begin{aligned}
    &\mathcal{E}^{\mathrm{train}}(\mathcal{N}; \xi):=
{ \frac{1}{N_\mathrm{train}}}\sum_{i=1}^{N_{\mathrm{train}}}|\mathcal{F}(\bu_i^{\mathrm{train}}-\mathcal{N}(\ba_i^{\mathrm{train}}))(\xi)|, \\
&\mathcal{E}^{\mathrm{test}}(\mathcal{N}; \xi):=
{ \frac{1}{N_\mathrm{test}}}\sum_{i=1}^{N_{\mathrm{test}}}|\mathcal{F}(\bu_i^{\mathrm{test}}-\mathcal{N}(\ba_i^{\mathrm{test}}))(\xi)|.
\end{aligned}
$$ 
where $\{\ba_i^{\mathrm{train}}, \bu_i^{\mathrm{train}}\}_{i=1}^{N_{\mathrm{train}}}$ and $\{\ba_i^{\mathrm{test}}, \bu_i^{\mathrm{test}}\}_{i=1}^{N_{\mathrm{test}}}$ are the training and testing datasets of Darcy rough task. Heuristically, for low frequencies $\xi$, $\mathcal{E}(\mathcal{N}; \xi)$ represents the capability of the neural network for predicting the ``global trend''. Conversely, for high frequencies $\xi$, $\mathcal{E}(\mathcal{N}; \xi)$ represents the capability for predicting variations on smaller scales. 
During training, we record $\mathcal{E}^{\mathrm{train}}(\mathcal{N}; \xi)$ and $\mathcal{E}^{\mathrm{test}}(\mathcal{N}; \xi)$ for each $\xi\in \mathbb{Z}_n^{2}$ at each epoch. 

From Figure \ref{fig:spectral_bias}, we conclude that existing methods struggle with learning higher frequencies. UNet and UNO mitigate this to some extent, likely due to their UNet-like multi-level structure. HANO's error pattern shows faster decay than others for higher frequencies and is more uniform overall. It also achieves lower testing errors. The mathematical heuristics of MWT, UNet, and UNO can be attributed to multigrid methods~\cite{XuZikatanov:2017} and wavelet-based multiresolution methods~\cite{Brewster1995,Beylkin1998}. However, these methods may have limitations for multiscale PDEs~\cite{Branets2009} because they apply instance-independent kernel integration regardless of the input data (or latent representations). In contrast, attention-based operations in the HANO architecture, which become more efficient when enabled by data-driven reduce/decompose operations, have the potential to address this limitation by adapting the kernel to a specific input instance.


\begin{figure}[H]
    \centering
    \subfigure[FNO ]{\includegraphics[width=0.19\textwidth]{  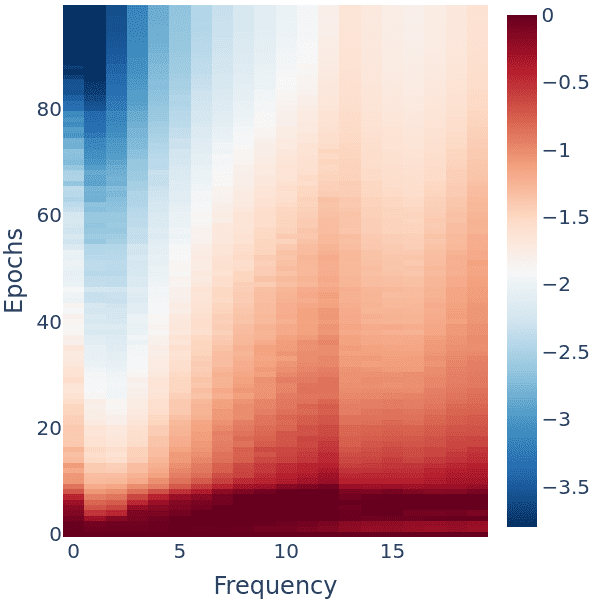}}
    \subfigure[FNO-$H^1$ ]{\includegraphics[width=0.19\textwidth]{  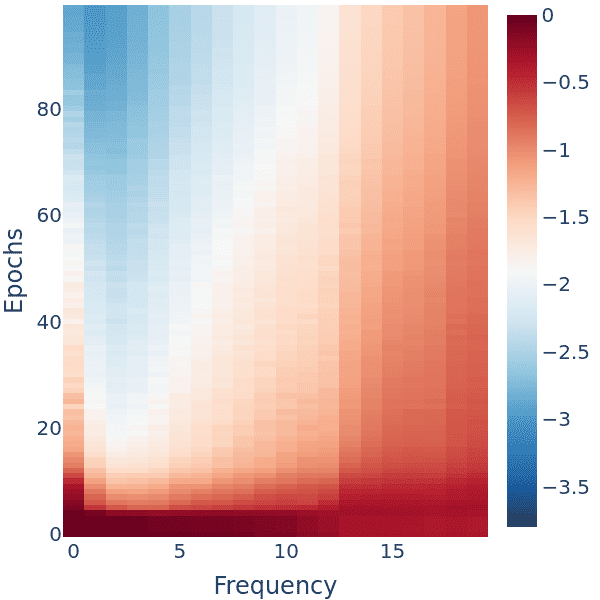}}
    \subfigure[MWT ]{\includegraphics[width=0.19\textwidth]{  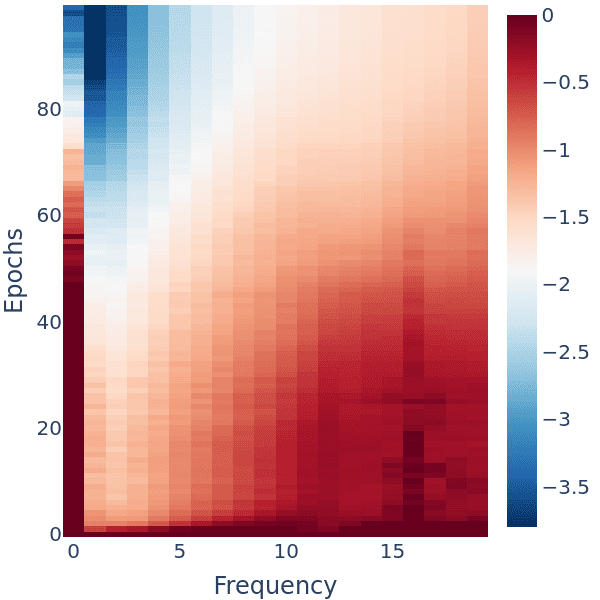}}
    \subfigure[UNO  ]{\includegraphics[width=0.19\textwidth]{  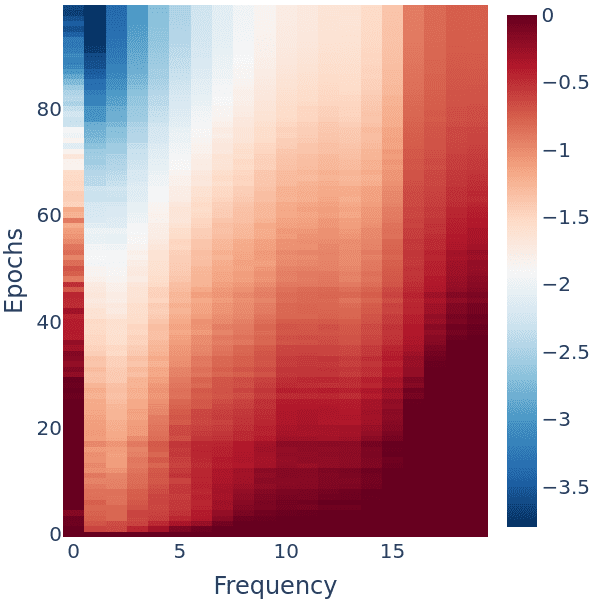}}
    \subfigure[HANO  ]{\includegraphics[width=0.19\textwidth]{  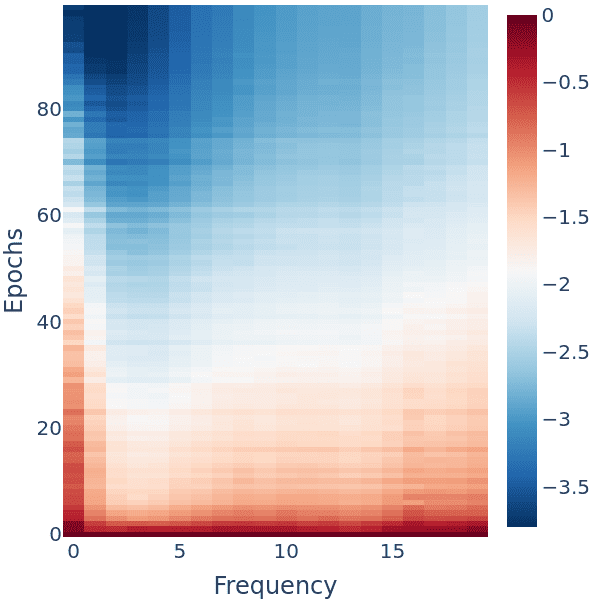}}
    \subfigure[FNO  ]{\includegraphics[width=0.19\textwidth]{  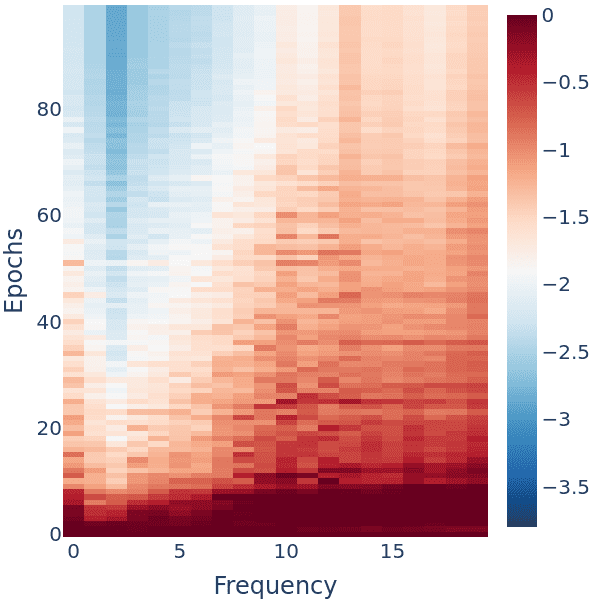}}
    \subfigure[FNO-$H^1$ ]{\includegraphics[width=0.19\textwidth]{  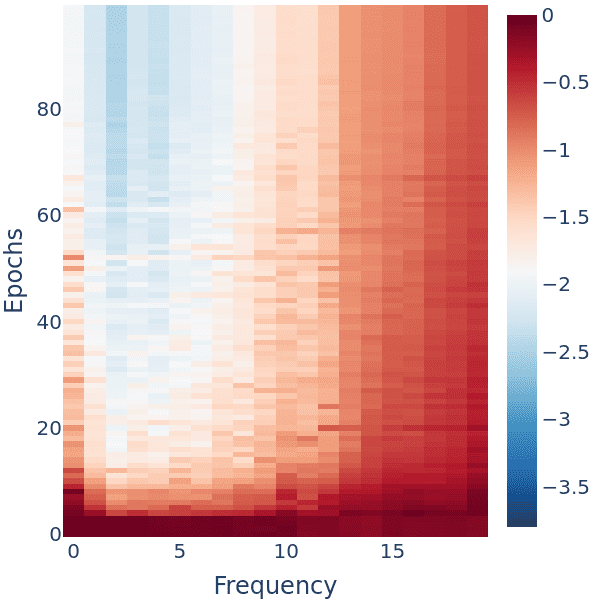}}
    \subfigure[MWT  ]{\includegraphics[width=0.19\textwidth]{  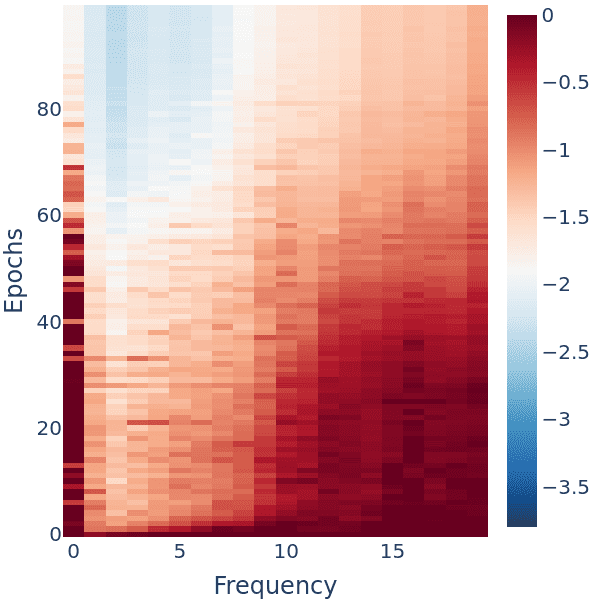}}
    \subfigure[UNO  ]{\includegraphics[width=0.19\textwidth]{  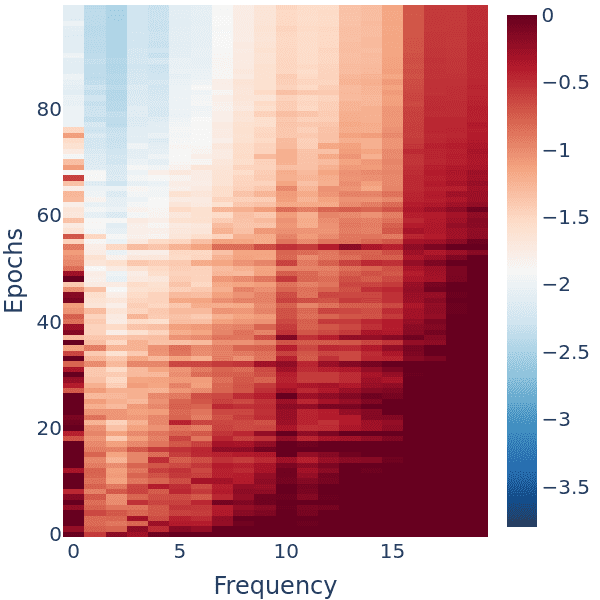}} 
    \subfigure[HANO  ]{\includegraphics[width=0.19\textwidth]{  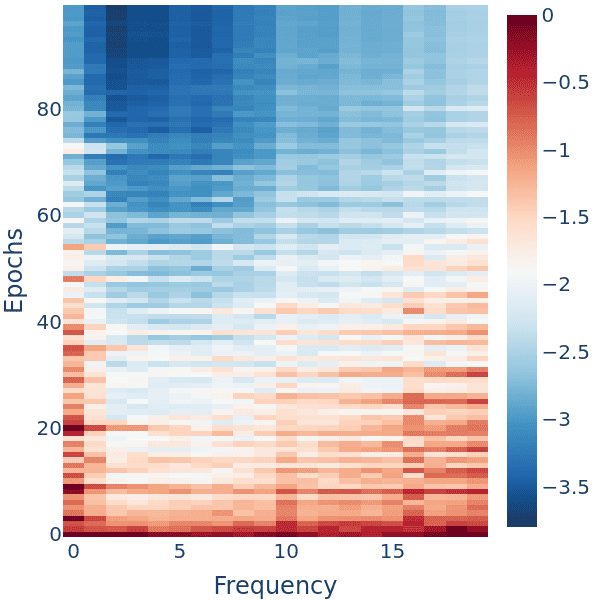}}
    \caption{\textbf{Top:} (a)-(e) show the training error dynamics in the frequency domain. The x-axis shows the first 20 dominating frequencies, from low frequency (left) to high frequency (right). The y-axis shows the number of training epochs. The colorbar shows the normalized $L^2$ error  (with respect to the error at epoch 0) in $\log_{10}$ scale. We compare five different methods; \textbf{Bottom:} (f)-(j) Corresponding testing error dynamics in the frequency domain for different methods. }
    \label{fig:spectral_bias}
\end{figure}

\paragraph{Comparison with Existing Methods}
Our comprehensive evaluation incorporates several contemporary methods:
\begin{itemize}
    \item FNO Variants:  The multiwavelet neural operator (MWT) \cite{gupta2021multiwavelet}, which implements wavelet convolutions on top of FNO's FFT architecture, is also included in this study. 
    \item UNet-based Models: We include the original UNet \cite{ronneberger2015u} and U-NO \cite{ashiqur2022u}, a U-shaped neural operator.
    \item Transformer-based Neural Operators: We also tested Galerkin Transformer (GT) \cite{cao2021choose}, and SWIN Transformer \cite{liu2021swin}, a general-architectured multiscale vision transformer. 
    \item $H^1$-loss evaluation ablation study: we train specific models (FNO2D, U-NO) using the $H^1$ loss to understand its effect, leading to the variants FNO2D $H^1$ and U-NO $H^1$.{   We have also trained a variant of HANO using the $L^2$ loss function, which we refer to as HANO $L^2$. This variant, along with the other variants we have discussed, is included in Table \ref{tab:bench}. }
\end{itemize}
In experiments, we found that \textsc{HANO} outperforms other neural operators in all tasks. 
The efficacy of the $H^1$ loss is noticeable across architectures; for instance, observe the performance difference between FNO2D, U-NO and their respective $H^1$-loss-trained variants. Additionally, the modifications in \textsc{FNO-cnn} notably elevate its performance with only a modest increase in runtime. As depicted in Figure \ref{fig:smooth&rough}, transitioning from Darcy smooth to Darcy rough and then to multiscale trigonometric problems, the enhancement in high-frequency components highlights \textsc{HANO}'s increasing advantage over other methods.

\paragraph{Comparison of Solutions/Derivatives for More Neural Operators}

We show the coefficient, reference solution from Multiscale trigonometric dataset, and the comparison with other operator learning models such as GT, SWIN, and MWT in Figure \ref{fig:gamblet_fig2}. HANO 
resolves the finer scale oscillations more accurately, as reflected by the predicted derivatives in (d) of Figure \ref{fig:gamblet_fig2}.
\begin{figure}[H]
    \centering 
    \subfigure[coefficient in $\log_{10}$ scale]{\includegraphics[width=0.24\textwidth]{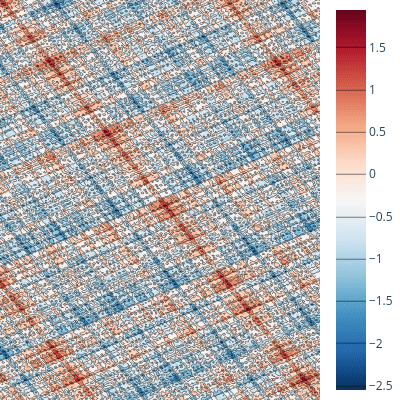}}
    \subfigure[reference solution in 2D]{\includegraphics[width=0.24\textwidth]{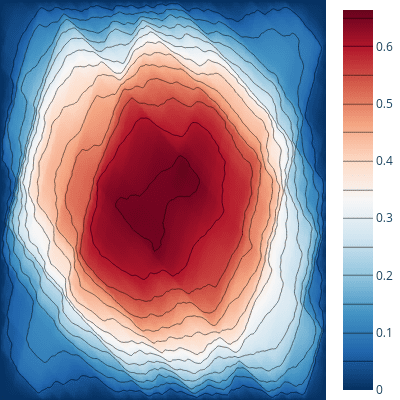}}
    \subfigure[1D slices of the predicted solutions]{\includegraphics[width=0.24\textwidth]{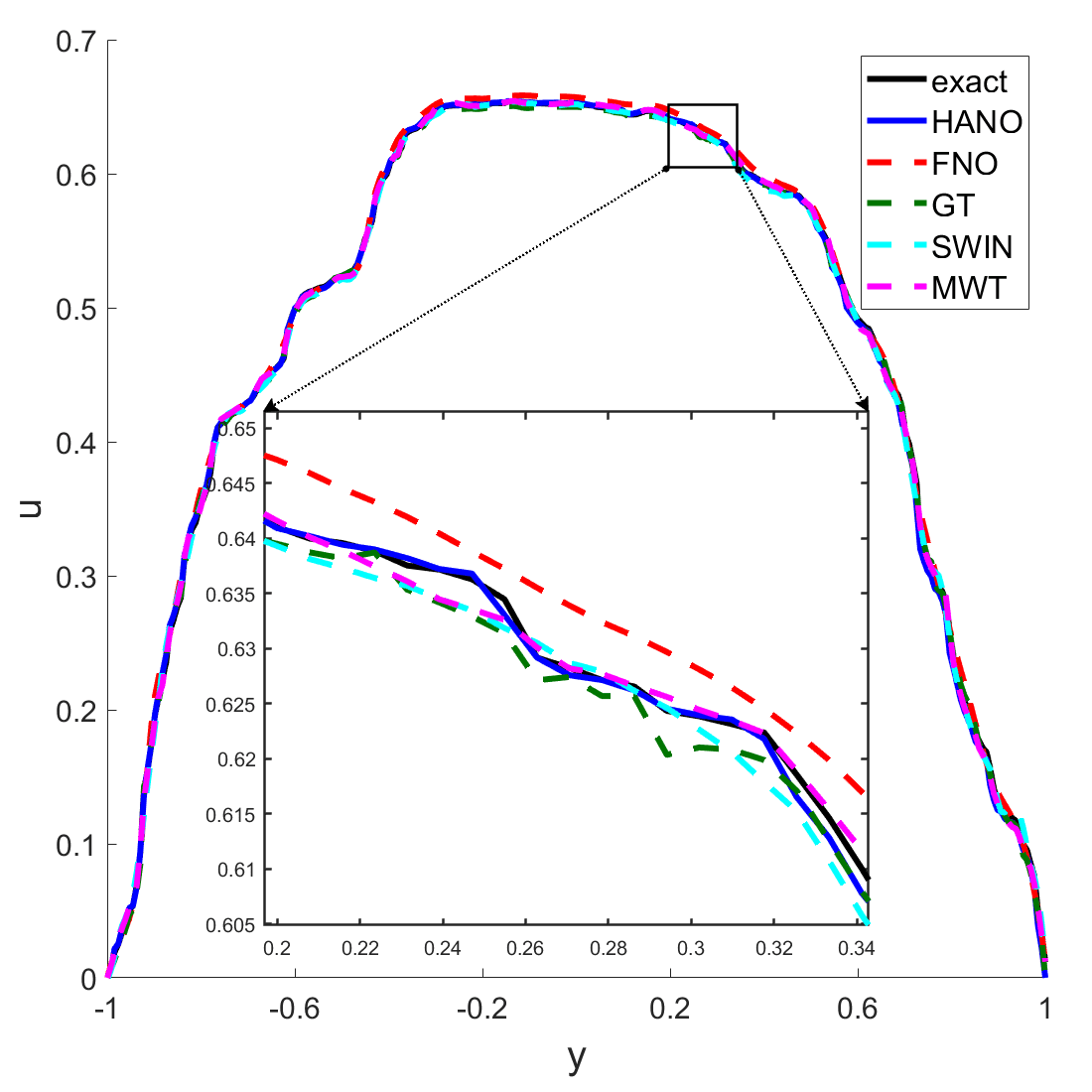}}
    \subfigure[1D slices of the predicted solution derivatives]{\includegraphics[width=0.24\textwidth]{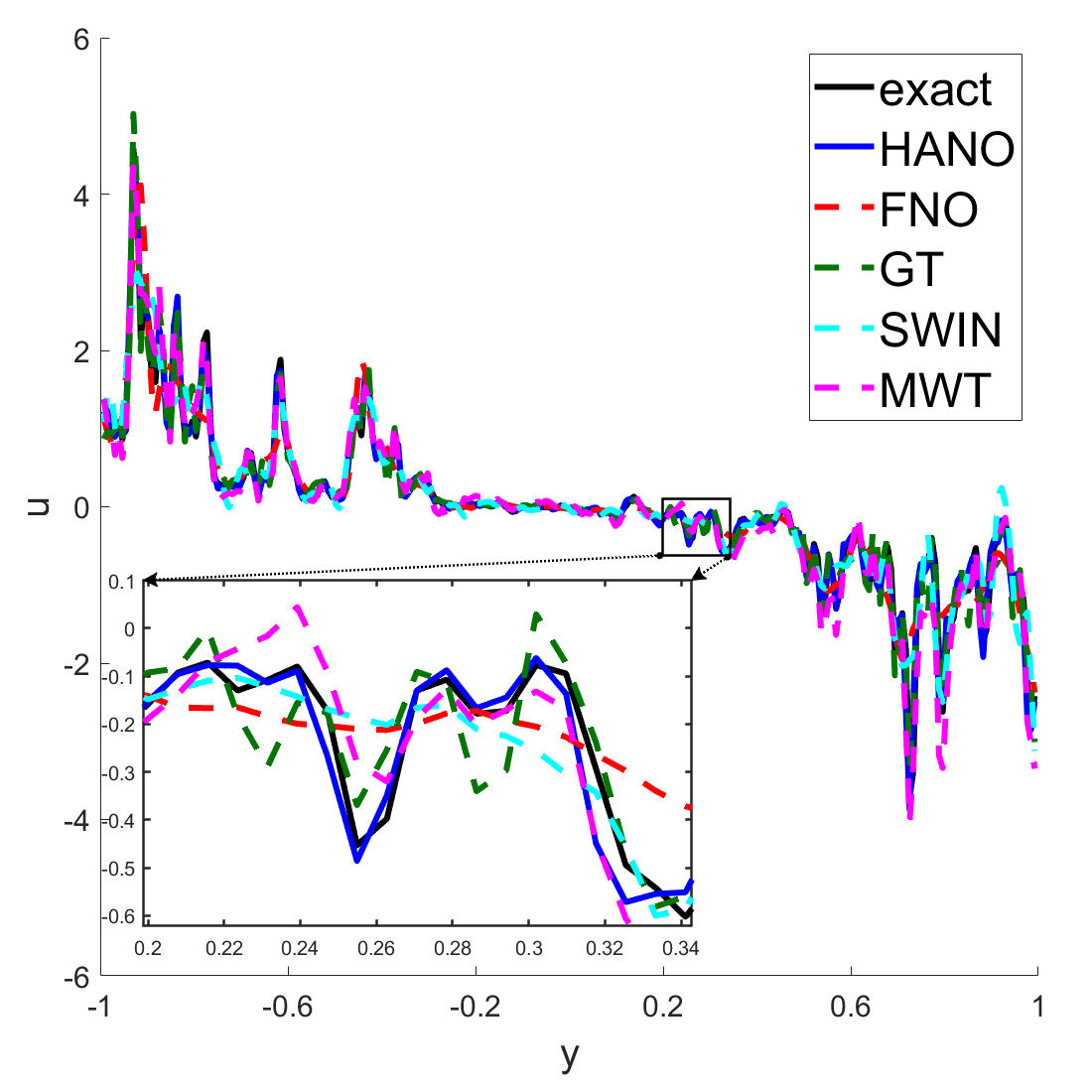}}
    \caption{(a) multiscale trigonometric coefficient, (b) reference solution, (c) comparison of predicted solutions on the slice $x=0$, (d) comparison of predicted derivative $\frac{\partial u}{\partial y}$ on the slice $x=0$.}
    \label{fig:gamblet_fig2}
\end{figure}

\paragraph{Comparison of Error Spectrum for More Neural Operators}

In Figs. \ref{fig:gamblet_fig} (c) and (d), we decompose the error into the frequency domain $[-256\pi, 256\pi]^2$ and plot the absolute error spectrum for HANO and FNO. Here, in \ref{fig:gamblet_error_freqerror}, we also include the absolute error and absolute error spectrum for other baseline models, such as MWT, GT, and SWIN. The comprehensive comparison also demonstrates that existing methods exhibit the phenomena of spectral bias to some degree. This empirical evidence also demonstrates the reason that HANO has the best accuracy in evaluation. 

\begin{figure}[htb]
    \centering
    \subfigure[FNO]{\includegraphics[width=0.19\textwidth]{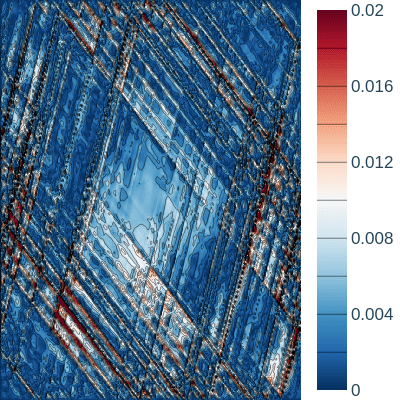}}
    \subfigure[MWT]{\includegraphics[width=0.19\textwidth]{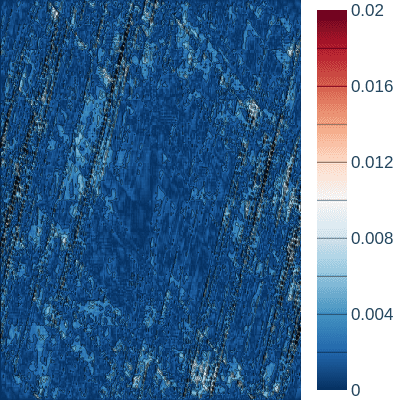}}
    \subfigure[GT]{\includegraphics[width=0.19\textwidth]{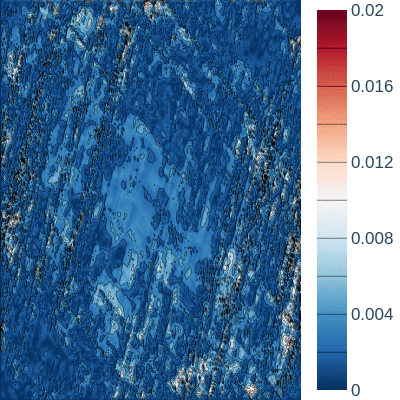}}
    \subfigure[SWIN]{\includegraphics[width=0.19\textwidth]{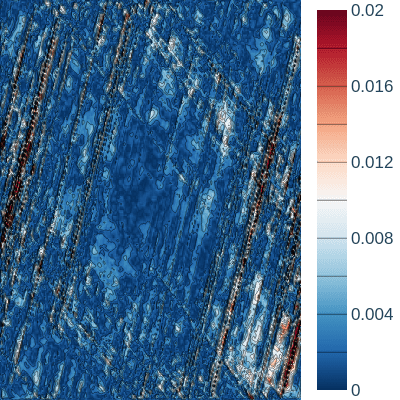}}
    \subfigure[HANO]{\includegraphics[width=0.19\textwidth]{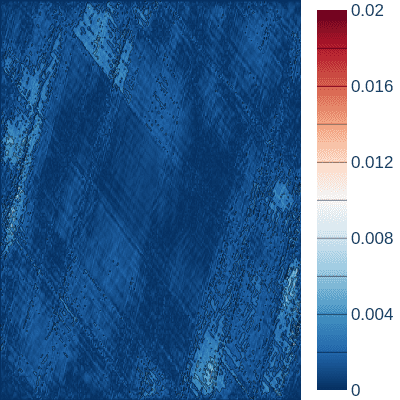}}
    \subfigure[FNO]{\includegraphics[width=0.19\textwidth]{figures/compressed/fno_freqerror.png}} 
    \subfigure[MWT]{\includegraphics[width=0.19\textwidth]{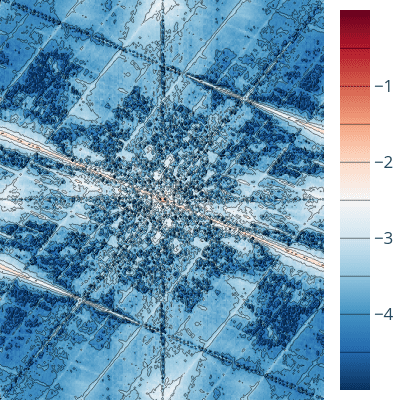}}
    \subfigure[GT]{\includegraphics[width=0.19\textwidth]{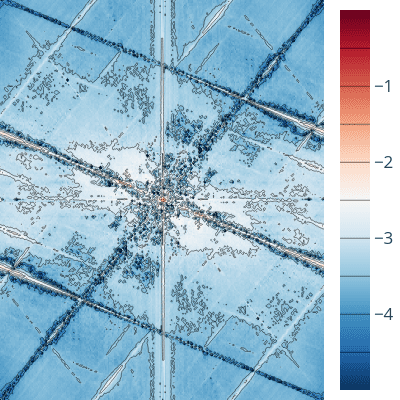}} 
    \subfigure[SWIN]{\includegraphics[width=0.19\textwidth]{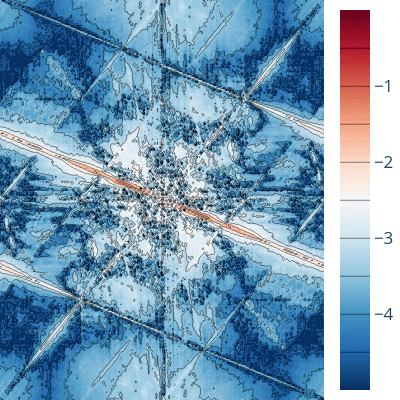}} 
    \subfigure[HANO]{\includegraphics[width=0.19\textwidth]{figures/compressed/hano_freqerror.png}}
    \caption{\textbf{Top:} (a)-(e) absolute error of different operator learning methods; \textbf{Bottom:} (f)-(j) absolute error spectrum in $\log_{10}$ scale of different opeator learning methods.}
    \label{fig:gamblet_error_freqerror}
\end{figure}

\subsection{Navier-Stokes Equation}
\label{sec:experiments:navierstokes}


In this section, we consider the 2D Navier-Stokes equation (NSE) dataset benchmarked in \cite{li2020fourier}. This dataset contains data generated for NSE in the vorticity-streamfunction formulation approximated by a pseudo-spectral solver with a Crank-Nicholson time stepping on the unit torus $\mathsf{T}^2$. For $x\in \mathsf{T}^2$ and $t \in [0, T]$, $\bm{u}(x,t)=\nabla^{\perp} \psi(x,t)$ is the velocity, and $\omega(x,t)$ denotes the vorticity. This formulation then writes
$$
\begin{aligned}
\partial_{t} \omega + \bm{u}  \cdot \nabla \omega  &=\nu \Delta \omega +f , & \\
\Delta \psi + \omega &= 0,
\\
\omega(x, 0) &=\omega_{0}(x), & 
\end{aligned}
$$
where $\omega_0$ is the initial vorticity field, $\nu$ is the viscosity, $f$ is the rotation of a vector forcing term, and $\mathrm{Re}$ is the Reynolds number, defined as $\mathrm{Re} := \frac{\rho u L}{\nu}$
with the density $\rho$ ($ =1$ here). The length scale of the fluid $L$ is set to $1$ here. The Reynolds number is a dimensionless parameter and is inversely proportional to the viscosity $\nu$. 
The operator to be learned is the approximation to 
$$\mathcal{S}: \omega(\cdot, 0\leq t\leq 9) \rightarrow \omega(\cdot, 10\leq t \leq T),$$ 
mapping the vorticity up to time 9 to the vorticity up to some later time $T$. We experiment with viscosities $\nu=10^{-3},10^{-4},10^{-5}$, and decrease the final time $T$ accordingly as the dynamics becomes more turbulent with increasing Reynolds number.

\paragraph{Time dependent neural operator}
Following the standard setup in \cite{li2020fourier}, we fix the resolution as $64 \times 64$ for both training and testing. Ten time-slices of solutions $\omega(\cdot, t)$ at $t=0, ..., 9$ are taken as the input data to the neural operator $\mathcal{N}$ which maps the solutions at 10 given timesteps to their subsequent time step. This procedure, often referred to as the rolled-out prediction, can be repeated recurrently until the final time $T$. For example, the $k$-th rollout is to obtain $\{\omega(\cdot, t_i)\}_{i=k-9}^{k} \mapsto \omega(\cdot, t_{k+1})$.
In Table \ref{tab:navier_tab}, the results are listed for HANO, FNO-3D (FFT in space-time), FNO-2D (FFT in space, and time rollouts), U-Net \cite{ronneberger2015u}, TF-Net \cite{wang2020towards}, ResNet \cite{he2016deep} and DilResNet \cite{stachenfeld2022learned}, and HANO achieves the best performance. 

Furthermore, we also test models on the same Navier-Stokes task with $\nu=10^{-5}$ but introduce an alternative training configuration labeled as $T=20$ (new). Note that all models incorporating these specialized training techniques show consistent performance enhancement compared to the original setup. Including these training tricks contributes to more stable generalization errors across various models, justifying their use in performance comparisons. Hence, both evaluation methods offer an equitable basis for contrasting the efficacy of our approach with existing baselines. We present a comprehensive comparison of the training configurations here. 
\paragraph{Two training setups}
\begin{itemize}
    \item \textbf{Original training setup:} Samples consist of 20 sequential time steps, with the goal of predicting the subsequent 10 time steps from the preceding 10. Using the roll-out prediction approach as described by \cite{li2020fourier}, the neural operator uses the initial 10-time steps to forecast the immediate next time step. This predicted time step is then merged with the prior 9 time steps to predict the ensuing time step. This iterative process continues to forecast the remaining 10 time steps.
    \item \textbf{New training setup:} Our findings indicate that an amalgamation of deep learning strategies is pivotal for the optimal performance of time-dependent tasks. While \cite{li2020fourier} employed the previous 10-time steps as inputs for the neural operator, our approach simplifies this. We find that leveraging just the current step's data, similar to traditional numerical solvers, is sufficient. During training, we avoid model unrolling. Originally, we had 1000 samples, each consisting of 20 sequential time steps. We have transformed these into 19,000 samples, where each sample now comprises a pair of sequential time steps. These are then shuffled and used to train the neural operator to predict the subsequent time step based on the current one. For testing, we revert to the roll-out prediction method as only the initial time step's ground truth is available. These methods align with some techniques presented in \cite{brandstetter2022message}. As shown in Table \ref{tab:navier_tab}, the second approach provides better performance. 
    Also note that FNO-3D performs FFT in both space and time, while other models are designed with more conventional marching-in-time schemes, such that they are applied in an autoregressive fashion. Therefore, the new alternative training configuration is not applicable to FNO-3D.

\end{itemize}


\begin{table*}
    \centering
    \begin{tabular}{l|c|ccccc} 
    \toprule
    &  &  $T=50 $ & $T=30$ & $T=30$   & $T=20$ & $T=20$ (new)\\
   \textbf{Model} & \#\textbf{Parameters} &  $\nu=1 \mathrm{e}-3$ & $\nu=1 \mathrm{e}-4$ & $\nu=1 \mathrm{e}-4$  & $\nu=1 \mathrm{e}-5$ & $\nu=1 \mathrm{e}-5$ \\
     & &  $N=1000$ & $N=1000$ & $N=10000$ &   $N=1000$ &   $N=1000$ \\
    \toprule FNO-3D & $6,558,537$ &  $0.0086$ & $0.1918$ & $0.0820$ & $0.1893$ & --- \\
    FNO-2D &  $2,368,001$ & $0.0128$ & $0.1559$ & $0.0834$   & $0.1556$ & $0.0624$ \\
    U-Net & $24,950,491$  & $0.0245$ & $0.2051$ & $0.1190$ &   $0.1982$ & $0.1058$ \\
    TF-Net & $7,451,724$ &  $0.0225$ & $0.2253$ & $0.1168$ &  $0.2268$ & $0.1241$  \\
    ResNet & $266,641$ &  $0.0701$ & $0.2871$ & $0.2311$ & $0.2753$ & $0.1518$ \\
    DilResNet  & $586,753$ & 0.0315 & 0.2561 & 0.2081 & 0.2315 & $0.1641$\\	
    \textsc{HANO} & $7,629,350$   & $\mathbf{0.0074}$  & $\mathbf{0.0557}$ &
    $\mathbf{0.0179}$ &  $\mathbf{0.0482}$ & $
    \mathbf{0.0265}$\\
    \toprule
    \end{tabular}
    \caption{Benchmark for the Navier-Stokes equation. $64 \times 64$ resolution is used for both training and testing. $N$ is the training sample size, either $N=1000$ or $N=10000$. The number of testing samples is 100 or 1000, respectively.  We use the $L^2$ loss function, the Adam optimizer, and the OneCycleLR scheduler with a cosine annealing strategy. The learning rate starts with $1\times 10^{-3}$ and decays to $1 \times 10^{-5}$. All models were trained for 500 epochs {  and use $L^2$ loss function}.}
    \label{tab:navier_tab}
\end{table*}
\begin{figure}[H]
    \centering
    \subfigure[Example rollout trajectories of the HANO model]{\includegraphics[width=.7\textwidth, height=.4\textwidth]{  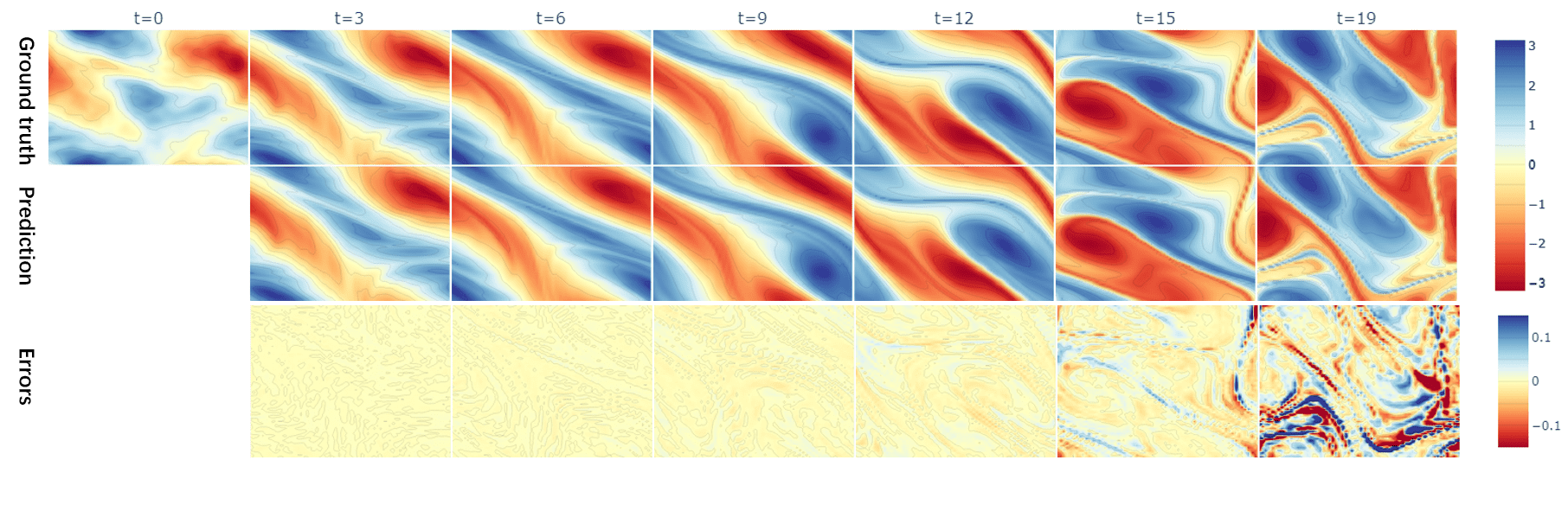}}
    \subfigure[{  Error vs. Time}]{\includegraphics[width=.28\textwidth, height=.26\textwidth]{  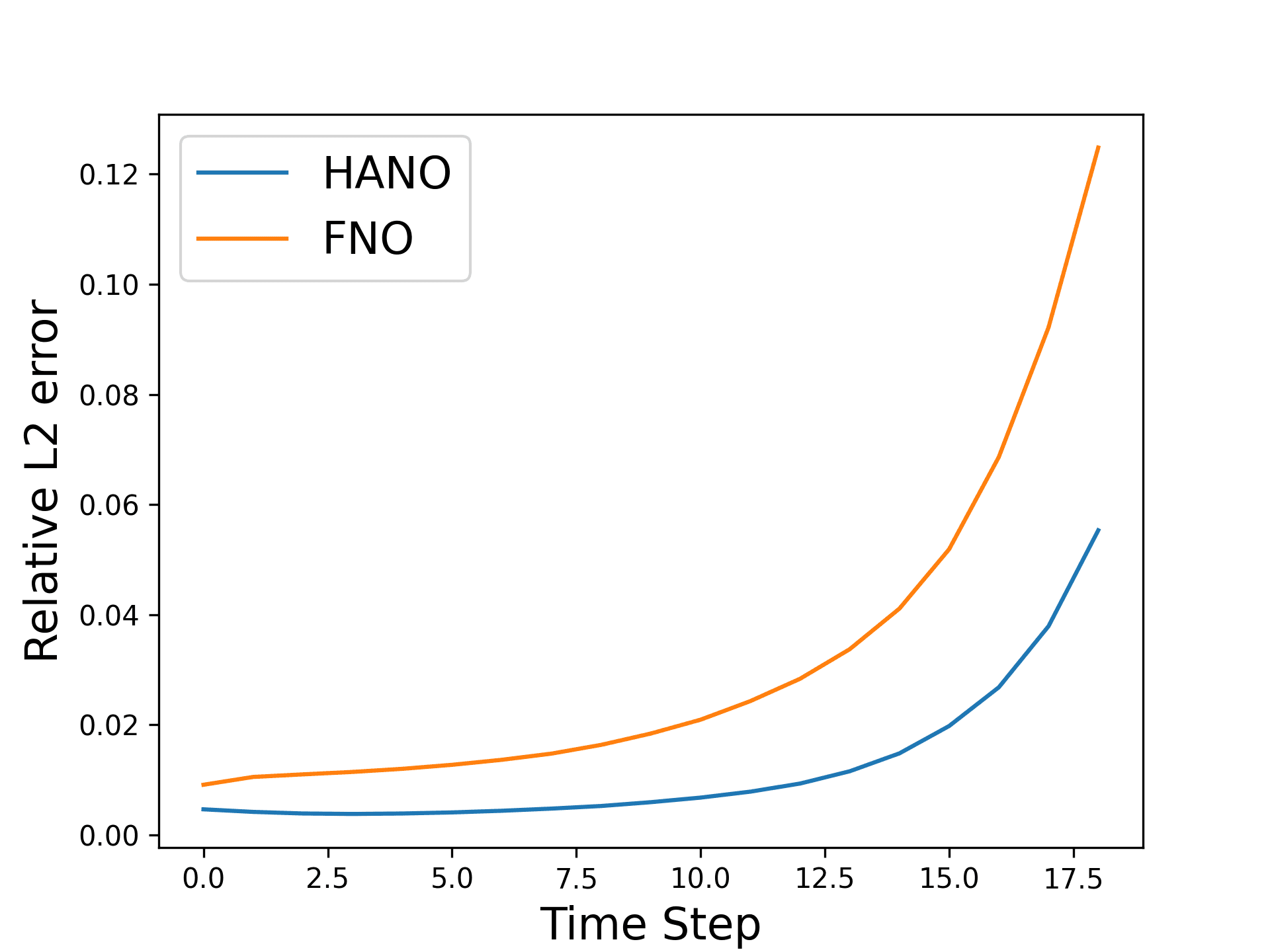}}
    \caption{Benchmark for the Navier-Stokes equation with $\nu = 1e-5$}
    \label{fig:Ns_roll}
\end{figure}
\subsection{Helmholtz equation}
\label{sec:experiments:helmholtz}

We test the performance of HANO for the acoustic Helmholtz equation in highly heterogeneous media as an example of multiscale wave phenomena, whose solution is considerably expensive for complicated and large geological models. This example and training data are taken from \cite{deHoop2022cost}, for the Helmholtz equation on the domain $\Omega:=(0,1)^2$. Given frequency $\omega=10^3$ and wavespeed field $c: \Omega \rightarrow \mathbb{R}$, the excitation field $u: \Omega \rightarrow \mathbb{R}$ solves the equation
$$
\left\{
\begin{aligned}
\left(-\Delta-\frac{\omega^2}{c^2(x)}\right) u &=0 & & \text { in } \Omega , \\
\frac{\partial u}{\partial n} &=0 & & \text { on } \partial \Omega_1, \partial \Omega_2, \partial \Omega_4, \\
\frac{\partial u}{\partial n} &=1 & & \text { on } \partial \Omega_3,
\end{aligned}
\right.
$$
where $\partial\Omega_3$ is the top side of the boundary, and $\partial\Omega_{1,2,4}$ are other sides. The wave speed field is $c(x)=20+\tanh (\tilde{c}(x))$, where $\tilde{c}$ is sampled from the Gaussian random field
$\tilde{c} \sim \mathcal{N}(0, \left(-\Delta+\tau^2\right)^{-d})$, 
where $\tau=3$ and $d=2$ are chosen to control the roughness. The Helmholtz equation is solved on a $100\times100$ grid by finite element methods. We aim to learn the mapping from  $\vc\in \mathbb{R}^{100\times100}$ to $\vu \in \mathbb{R}^{100\times100}$ as shown in Figure \ref{fig:helm_a_u}.
In this example, following the practice in \cite{deHoop2022cost}, a  training dataset of size $4000$ examples is adopted, while the test dataset contained $800$ examples. All models were trained for 100 epochs.

\begin{figure}[H]
    \centering
    \subfigure[wavespeed field $\vc$]{\includegraphics[width=0.3\textwidth]{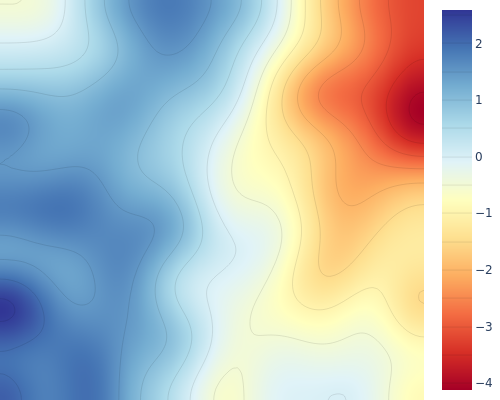}}
    \subfigure[groud truth $\vu$]{\includegraphics[width=0.3\textwidth]{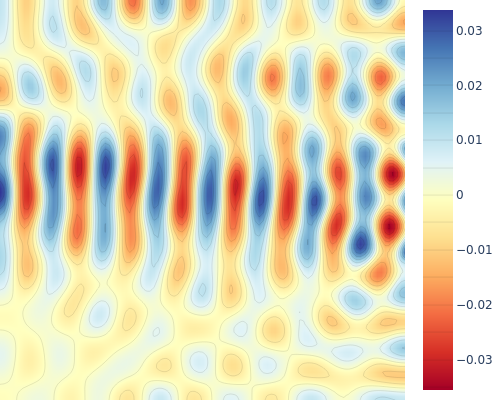}}
    \subfigure[HANO prediction $\hat{\vu}$]{\includegraphics[width=0.3\textwidth]{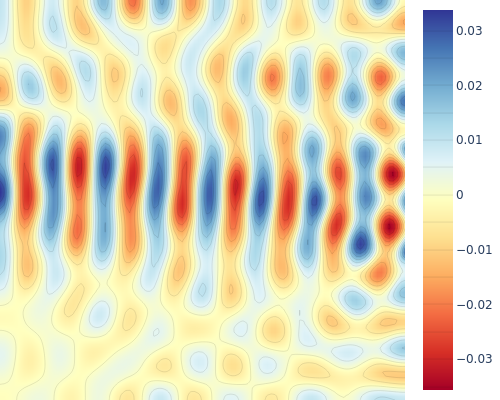}}
    \caption{The mapping $\vc \mapsto \vu$}
    \label{fig:helm_a_u}
\end{figure}


\begin{table}[H]

 \caption{Performance on the Helmholtz benchmark.}
    \centering
    \begin{tabular}{lrrrr}
    \toprule
      \textbf{Model}  &time  &params(m)  & $L^2 (\times 10^{-2})$ &$H^1 (\times 10^{-2})$\\
      \midrule
       \textsc{FNO2D} &  6.2 &  16.93  & 1.25 &7.66\\
       \textsc{UNet} & 7.1  & 17.26   & 3.81& 23.31\\
       \textsc{DilResNet} &10.8  &1.03 & 4.34 &34.21 \\
       \textsc{U-NO} & 21.5  &  16.39  & 1.26 &8.03\\
       \textsc{LSM} & 28.2  &   4.81 & 2.55 &10.61\\
        \textsc{HANO} & 13.1 &  11.35  & $\bm{0.95}$ & $\bm{6.10}$ \\
    \bottomrule
    \end{tabular}

    \label{tab:helmholtz}
\end{table}

 { For issues like high-frequency problems, especially Helmholtz equations with a large wavenumber, HANO might exhibit limitations due to the inherent challenge in approximating high-frequency components locally. Such scenarios highlight the difficulty of capturing the propagation of high-frequency solutions through local attention mechanisms alone. This challenge is analogous to the limitations faced by classical numerical methods in devising straightforward hierarchical matrix formulations for the Green's function of Helmholtz equations with large wavenumbers.  It is noteworthy that the wavenumber is modest in this benchmark, and HANO surpasses other baseline methods, by a relatively modest margin as shown in Table \ref{tab:helmholtz}. } The error of our rerun FNO is comparable with the reported results in \cite{deHoop2022cost}. We note that the four models benchmarked for the Helmholtz equation in \cite{deHoop2022cost}, including FNO and DeepONet, failed to reach a relative error less than $1\times 10^{-2}$. We also compare the evaluation time of the trained models in Table \ref{tab:helmholtz}. Compared to HANO, FNO has both a larger error and takes longer to evaluate. 
Moreover, FNO is known to be prone to overfitting the data when increasing the stacking of spectral convolution layers deeper and deeper \cite{tran2023factorized}. UNet, as a CNN-based method, can evaluate much faster (30 times faster than HANO) but has the worst error (60 times higher than HANO).

\subsection{Training Dynamics}
\label{sec:experiments:trainingdynamics}

We present the dynamics of training and testing error over 100 epochs of training in Figure \ref{fig:converge}. We compare HANO trained with $H^1$ and $L^2$ loss functions, and show the evolution of errors as well as the loss curves during the training process. The comparison shows that HANO with $H^1$ loss achieves lower training and testing errors. It also suggests that the  $H^1$ loss function reduces the generalization gap (measured by the difference between training error and testing error), while $L^2$ loss function fails to do so.

\begin{figure}[H]
    \centering
    \subfigure[Training curve by using $H^1$ loss]
    {\includegraphics[width=0.35\textwidth]{  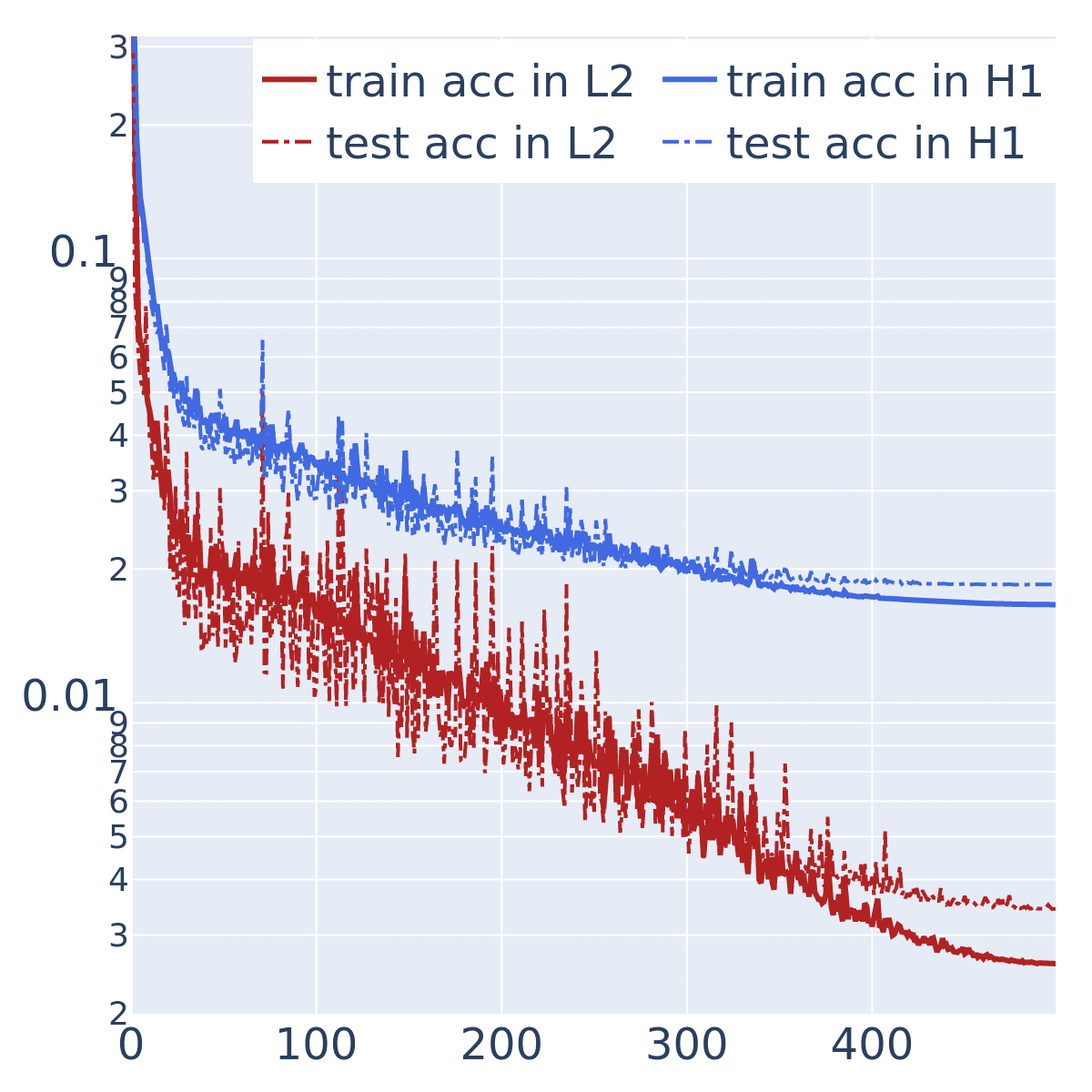}}\qquad
    \subfigure[Training curve by using $L^2$ loss]
    {\includegraphics[width=0.35\textwidth]{  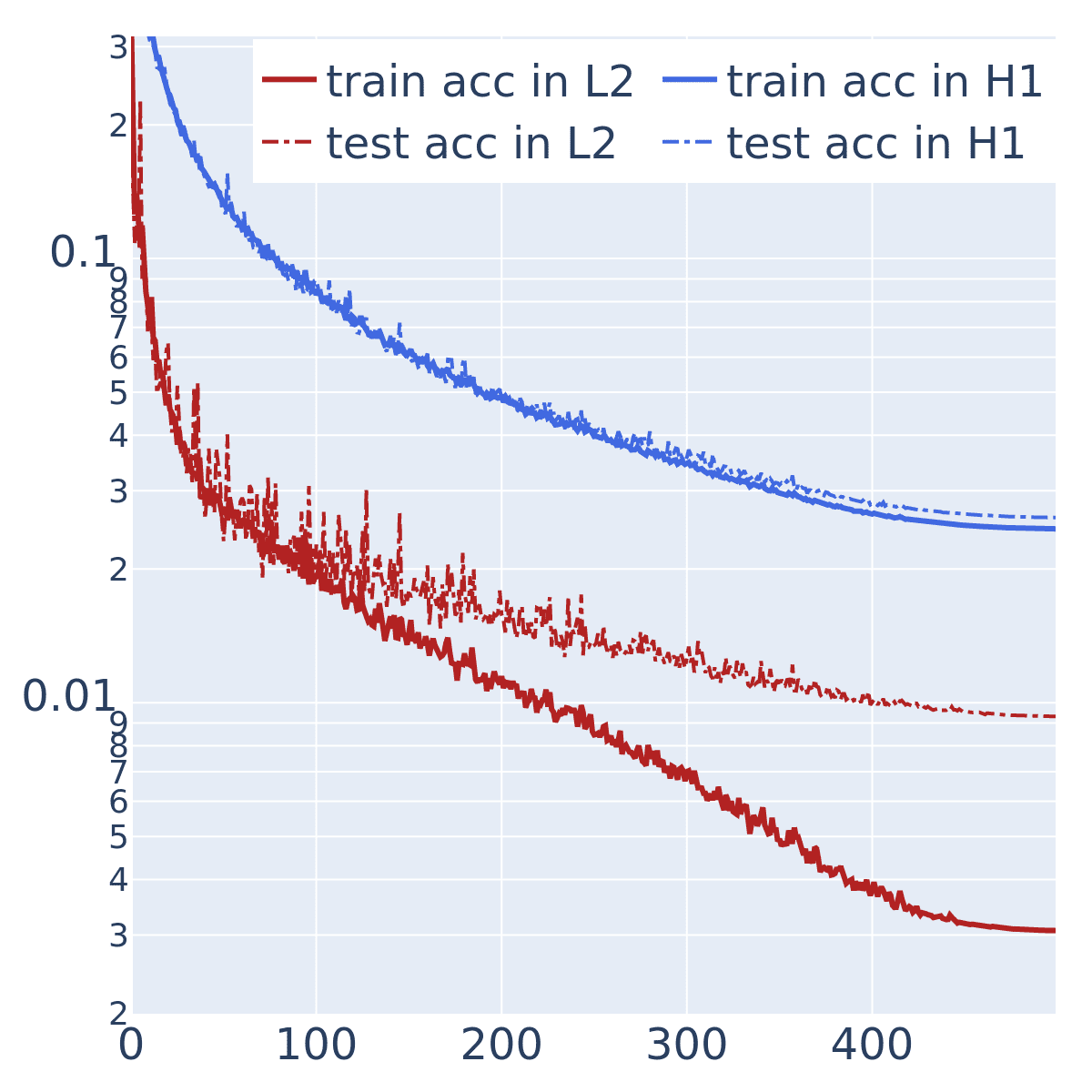}}
    \caption{Comparison of training dynamics between HANO trained with $H^1$ loss and $L^2$ loss}
    \label{fig:converge}
\end{figure}

\subsection{Memory Usage} 

We report the memory usage of different models for the Darcy smooth (with resolution 211$\times$211) and Darcy rough (with resolution 256$\times$256) benchmarks in Table \ref{tab:memory}. The table shows that the memory usage of HANO remains stable across resolutions. For the higher resolution of 256$\times$256, both MWT and GT consume more CUDA memory than HANO, even though HANO achieves much higher accuracy.

\begin{table}[ht]
    \begin{center}
    \small
    \begin{tabular}{llrrr}
    \toprule
    &\textbf{Model} &  \textbf{Darcy smooth} & \textbf{Darcy rough}  \\ 
    \midrule
     &\textsc{FNO} 
    & $0.72$  & $1.04$\\
    &\textsc{GT} 
    & $1.40$  & $4.53$ \\
     &\textsc{MWT} 
    & ---   & $1.27$ \\
    &\textsc{HANO} 
    & $0.89$  & $1.21$\\

    \bottomrule\\[-2.5mm]
    \end{tabular}   
    \end{center}
    \caption{The memory usage(GB) of different models by using torchinfo.}
    \label{tab:memory}
\end{table}





\subsection{Discretization invariance}
\label{sec:appendix:discretizationinvariance}

FNO achieves discretization invariance through Fourier interpolation, enabling models trained on low-resolution data to handle high-resolution input. By incorporating suitable interpolation operators, HANO can achieve a comparable capability. Specifically, it can be trained on lower resolution data but evaluated at higher resolution, without requiring any higher resolution data during training (achieving zero-shot super-resolution). 

We conducted the experiments following the same setup as in \cite{li2020fourier} for the multiscale trigonometric coefficient benchmark. The models were trained on $64\times64$, $128\times128$, and $256\times256$ resolutions, and tested on $128\times128$, $256\times256$, and $512\times512$ resolutions, respectively. Results in Table \ref{tab:Discretization_invariance} show that HANO incorporating linear interpolation is more stable than FNO.

\begin{table}[ht]
    \begin{center}
    \small
    \begin{tabular}{llrrrrrrr}
    \toprule
    &  & \multicolumn{3}{c}{\textbf{FNO}} & \multicolumn{3}{c}{\textbf{HANO}} \\ \cmidrule(lr){1-2}\cmidrule(lr){3-5}\cmidrule(lr){6-8}
    & \diagbox{Train}{Test} &  $128$ & $256$ & $512$ &  $128$ & $256$ & $512$ \\ 
    \midrule
    & $64$
    & $5.2808$ & $7.9260$ & $9.1054$
    & $1.3457$ & $1.3557$ & $1.3624$ \\
    & $128$
    &  & $3.9753$ & $6.0156$
    &  & $0.6715$ & $0.6835$ \\
    & $256$ 
    &  & & $3.1871$
    &  &  & $0.5941$ \\
    \hline
    \end{tabular}
    \end{center}
    \caption{Comparison of discretization invariance property for HANO and FNO for the multiscale trigonometric coefficient benchmark. The relative $L^2$ error ($\times 10^{-2}$) with respect to the reference solution on the testing resolution is measured.}
    \label{tab:Discretization_invariance}
\end{table}

\subsection{Datasets and Code}
{  
The code and datasets can be accessed at the following location \url{https://github.com/xlliu2017/HANO}.
}
\section{Conclusion}
\label{sec:conclusion}

In this work, we investigated the ``spectral bias'' phenomenon commonly observed in multiscale operator learning. To our best knowledge, we conducted the first in-depth numerical study of this issue. We proposed HANO, a hierarchical attention-based model to mitigate the spectral bias. HANO employs a fine-coarse-fine V-cycle update and an empirical $H^1$ loss to recover fine-scale features in the multiscale solutions. Our experiments show that HANO outperforms existing neural operators on multiscale benchmarks in terms of accuracy and robustness. 

\emph{Limitation and outlook}: (1) HANO's current implementation requires a regular grid, and extending it to data clouds and graph neural networks could offer new opportunities to exploit its hierarchical representation. (2) The current attention-based operator in HANO can achieve discretization invariance using simple interpolation \cite{2023JMLRNeural} (e.g. see Section \ref{sec:appendix:discretizationinvariance}). However, either simple interpolation or Fourier interpolation (used by FNO) may suffer from aliasing errors in the frequency domain, as indicated by our experiments and recent analysis \cite{kovachki2021universal,lanthaler2022error}. Better balance between discretization invariance and model accuracy may be achieved with proper operator-adaptive sampling and interpolation techniques.


\section*{Acknowledgments}
BX and LZ are partially supported by the Shanghai Municipal Science and Technology Project 22JC1401600, China's National Key Research and Development Projects (2023YFF0805200), NSFC grant 12271360, and the Fundamental Research Funds for the Central Universities. SC is partially supported by the National Science Foundation award DMS-2309778.


\newpage

\bibliography{ref}

\begin{thebibliography}{10}
\expandafter\ifx\csname url\endcsname\relax
  \def\url#1{\texttt{#1}}\fi
\expandafter\ifx\csname urlprefix\endcsname\relax\def\urlprefix{URL }\fi
\expandafter\ifx\csname href\endcsname\relax
  \def\href#1#2{#2} \def\path#1{#1}\fi

\bibitem{Branets2009}
L.~V. Branets, S.~S. Ghai, L.~L., X.-H. Wu, Challenges and technologies in reservoir modeling, Commun. Comput. Phys. 6~(1) (2009) 1--23.

\bibitem{Engquist2008}
B.~Engquist, P.~E. Souganidis, Asymptotic and numerical homogenization, Acta Numerica 17 (2008) 147--190.

\bibitem{Hou1999}
T.~Y. Hou, X.-H. Wu, Z.~Cai, Convergence of a multiscale finite element method for elliptic problems with rapidly oscillating coefficients, Math. Comp. 68~(227) (1999) 913--943.

\bibitem{eh09}
Y.~Efendiev, T.~Hou, {Multiscale Finite Element Methods: Theory and Applications}, Vol.~4, Springer Science \& Business Media, 2009.

\bibitem{Efendiev2013}
Y.~Efendiev, J.~Galvis, T.~Y. Hou, \href{http://www.sciencedirect.com/science/article/pii/S0021999113003392}{Generalized multiscale finite element methods (gmsfem)}, Journal of Computational Physics 251 (2013) 116 -- 135.
\newblock \href {https://doi.org/http://dx.doi.org/10.1016/j.jcp.2013.04.045} {\path{doi:http://dx.doi.org/10.1016/j.jcp.2013.04.045}}.
\newline\urlprefix\url{http://www.sciencedirect.com/science/article/pii/S0021999113003392}

\bibitem{chung2016adaptive}
E.~Chung, Y.~Efendiev, T.~Y. Hou, Adaptive multiscale model reduction with generalized multiscale finite element methods, Journal of Computational Physics 320 (2016) 69--95.

\bibitem{chung2023multiscale}
E.~Chung, Y.~Efendiev, T.~Y. Hou, Multiscale Model Reduction: Multiscale Finite Element Methods and Their Generalizations, Vol. 212, Springer Nature, 2023.

\bibitem{Hackbusch1985}
W.~Hackbusch, Multigrid Methods and Applications, Vol.~4 of Springer Series in Computational Mathematics, Springer-Verlag, Berlin, 1985.

\bibitem{XuZikatanov:2017}
J.~Xu, L.~Zikatanov, Algebraic multigrid methods, Acta Numerica 26 (2017) 591--721.

\bibitem{Brewster1995}
M.~E. Brewster, G.~Beylkin, A multiresolution strategy for numerical homogenization, Appl. Comput. Harmon. Anal. 2~(4) (1995) 327--349.

\bibitem{Beylkin1998}
G.~Beylkin, N.~Coult, A multiresolution strategy for reduction of elliptic {PDE}s and eigenvalue problems, Appl. Comput. Harmon. Anal. 5~(2) (1998) 129--155.

\bibitem{OwhadiMultigrid:2017}
H.~Owhadi, Multigrid with {R}ough {C}oefficients and {M}ultiresolution {O}perator {D}ecomposition from {H}ierarchical {I}nformation {G}ames, SIAM Rev. 59~(1) (2017) 99--149.

\bibitem{Greengard1987}
L.~Greengard, V.~Rokhlin, A fast algorithm for particle simulations, J. Comput. Phys. 73~(2) (1987) 325--348.

\bibitem{Hackbusch2002}
W.~Hackbusch, L.~Grasedyck, S.~B{\"o}rm, An introduction to hierarchical matrices, in: Proceedings of Equadiff 10, Masaryk University, 2002, pp. 101--111.

\bibitem{Ho2016}
K.~Ho, L.~Ying, Hierarchical interpolative factorization for elliptic operators: Differential equations, Communiations on Pure and Applied Mathematics 69~(8) (2016) 1415--1451.

\bibitem{Bebendorf2005}
M.~Bebendorf, Efficient inversion of the galerkin matrix of general second-order elliptic operators with nonsmooth coefficients, Math. Comp. 74~(251) (2005) 1179--1199.

\bibitem{ZhuZabaras:2018}
Y.~Zhu, N.~Zabaras, Bayesian deep convolutional encoder--decoder networks for surrogate modeling and uncertainty quantification., J. Comput. Phys. 366 (2018) 415--447.

\bibitem{Fan2019}
Y.~Fan, J.~Feliu-Fab\'{a}, L.~Lin, L.~Ying, L.~Zepeda-Nunez, A multiscale neural network based on hierarchical nested bases, Res. Math. Sci. 6~(21) (2019).

\bibitem{fan2019multiscale}
Y.~Fan, L.~Lin, L.~Ying, L.~Zepeda-N{\'u}nez, A multiscale neural network based on hierarchical matrices, Multiscale Modeling \& Simulation 17~(4) (2019) 1189--1213.

\bibitem{Khoo2020}
Y.~Khoo, J.~Lu, L.~Ying., Solving parametric pde problems with artificial neural networks, European Journal of Applied Mathematics 32~(3) (2020) 421--435.

\bibitem{lu2021learning}
L.~Lu, P.~Jin, G.~Pang, Z.~Zhang, G.~E. Karniadakis, Learning nonlinear operators via deeponet based on the universal approximation theorem of operators, Nature Machine Intelligence 3~(3) (2021) 218--229.

\bibitem{chen1995universal}
T.~Chen, H.~Chen, Universal approximation to nonlinear operators by neural networks with arbitrary activation functions and its application to dynamical systems, IEEE transactions on neural networks 6~(4) (1995) 911--917.

\bibitem{li2020fourier}
Z.~Li, N.~Kovachki, K.~Azizzadenesheli, B.~Liu, K.~Bhattacharya, A.~Stuart, A.~Anandkumar, Fourier neural operator for parametric partial differential equations, The International Conference on Learning Representations (2021).

\bibitem{gupta2021multiwavelet}
G.~Gupta, X.~Xiao, P.~Bogdan, Multiwavelet-based operator learning for differential equations, Advances in Neural Information Processing Systems 34 (2021) 24048--24062.

\bibitem{brandstetter2022message}
J.~Brandstetter, D.~E. Worrall, M.~Welling, \href{https://openreview.net/forum?id=vSix3HPYKSU}{Message passing neural {PDE} solvers}, in: International Conference on Learning Representations, 2022.
\newline\urlprefix\url{https://openreview.net/forum?id=vSix3HPYKSU}

\bibitem{seidman2022nomad}
J.~Seidman, G.~Kissas, P.~Perdikaris, G.~J. Pappas, Nomad: Nonlinear manifold decoders for operator learning, Advances in Neural Information Processing Systems 35 (2022) 5601--5613.

\bibitem{chen2021solving}
Y.~Chen, B.~Hosseini, H.~Owhadi, A.~M. Stuart, Solving and learning nonlinear pdes with gaussian processes, Journal of Computational Physics 447 (2021) 110668.

\bibitem{brandstetter2023clifford}
J.~Brandstetter, R.~van~den Berg, M.~Welling, J.~K. Gupta, \href{https://openreview.net/forum?id=okwxL_c4x84}{Clifford neural layers for {PDE} modeling}, in: The Eleventh International Conference on Learning Representations, 2023.
\newline\urlprefix\url{https://openreview.net/forum?id=okwxL_c4x84}

\bibitem{stachenfeld2022learned}
K.~Stachenfeld, D.~B. Fielding, D.~Kochkov, M.~Cranmer, T.~Pfaff, J.~Godwin, C.~Cui, S.~Ho, P.~Battaglia, A.~Sanchez-Gonzalez, Learned simulators for turbulence, in: International Conference on Learning Representations, 2022.

\bibitem{vaswani2017attention}
A.~Vaswani, N.~Shazeer, N.~Parmar, J.~Uszkoreit, L.~Jones, A.~N. Gomez, L.~Kaiser, I.~Polosukhin, Attention is all you need, Advances in neural information processing systems 30 (2017).

\bibitem{brown2020language}
T.~Brown, B.~Mann, N.~Ryder, M.~Subbiah, J.~D. Kaplan, P.~Dhariwal, A.~Neelakantan, P.~Shyam, G.~Sastry, A.~Askell, et~al., Language models are few-shot learners, Advances in neural information processing systems 33 (2020) 1877--1901.

\bibitem{dosovitskiy2020image}
A.~Dosovitskiy, L.~Beyer, A.~Kolesnikov, D.~Weissenborn, X.~Zhai, T.~Unterthiner, M.~Dehghani, M.~Minderer, G.~Heigold, S.~Gelly, J.~Uszkoreit, N.~Houlsby, \href{https://openreview.net/forum?id=YicbFdNTTy}{An image is worth 16x16 words: Transformers for image recognition at scale}, in: International Conference on Learning Representations, 2021.
\newline\urlprefix\url{https://openreview.net/forum?id=YicbFdNTTy}

\bibitem{ho2020denoising}
J.~Ho, A.~Jain, P.~Abbeel, Denoising diffusion probabilistic models, Advances in neural information processing systems 33 (2020) 6840--6851.

\bibitem{rombach2022high}
R.~Rombach, A.~Blattmann, D.~Lorenz, P.~Esser, B.~Ommer, High-resolution image synthesis with latent diffusion models, in: Proceedings of the IEEE/CVF conference on computer vision and pattern recognition, 2022, pp. 10684--10695.

\bibitem{cao2021choose}
S.~Cao, Choose a transformer: Fourier or galerkin, Advances in Neural Information Processing Systems 34 (2021).

\bibitem{geneva2022transformers}
N.~Geneva, N.~Zabaras, Transformers for modeling physical systems, Neural Networks 146 (2022) 272--289.

\bibitem{kissas2022learning}
G.~Kissas, J.~H. Seidman, L.~F. Guilhoto, V.~M. Preciado, G.~J. Pappas, P.~Perdikaris, Learning operators with coupled attention, The Journal of Machine Learning Research 23~(1) (2022) 9636--9698.

\bibitem{li2023transformer}
Z.~Li, K.~Meidani, A.~B. Farimani, \href{https://openreview.net/forum?id=EPPqt3uERT}{Transformer for partial differential equations' operator learning}, Transactions on Machine Learning Research (2023).
\newline\urlprefix\url{https://openreview.net/forum?id=EPPqt3uERT}

\bibitem{hao2023gnot}
Z.~Hao, Z.~Wang, H.~Su, C.~Ying, Y.~Dong, S.~Liu, Z.~Cheng, J.~Song, J.~Zhu, {GNOT: A general neural operator transformer for operator learning}, in: International Conference on Machine Learning, PMLR, 2023, pp. 12556--12569.

\bibitem{pmlr-v202-de-oliveira-fonseca23a}
A.~H. De~Oliveira~Fonseca, E.~Zappala, J.~Ortega~Caro, D.~V. Dijk, Continuous spatiotemporal transformer, in: A.~Krause, E.~Brunskill, K.~Cho, B.~Engelhardt, S.~Sabato, J.~Scarlett (Eds.), Proceedings of the 40th International Conference on Machine Learning, Vol. 202 of Proceedings of Machine Learning Research, PMLR, 2023, pp. 7343--7365.

\bibitem{xiao2023improved}
Z.~Xiao, Z.~Hao, B.~Lin, Z.~Deng, H.~Su, Improved operator learning by orthogonal attention (2023).
\newblock \href {http://arxiv.org/abs/2310.12487} {\path{arXiv:2310.12487}}.

\bibitem{2023JMLRNeural}
N.~Kovachki, Z.~Li, B.~Liu, K.~Azizzadenesheli, K.~Bhattacharya, A.~Stuart, A.~Anandkumar, \href{http://jmlr.org/papers/v24/21-1524.html}{Neural operator: Learning maps between function spaces with applications to pdes}, Journal of Machine Learning Research 24~(89) (2023) 1--97.
\newline\urlprefix\url{http://jmlr.org/papers/v24/21-1524.html}

\bibitem{bartolucci2023neural}
F.~Bartolucci, E.~de~B{\'e}zenac, B.~Raoni{\'c}, R.~Molinaro, S.~Mishra, R.~Alaifari, Are neural operators really neural operators? frame theory meets operator learning, arXiv preprint arXiv:2305.19913 (2023).

\bibitem{liu2023nuno}
S.~Liu, Z.~Hao, C.~Ying, H.~Su, Z.~Cheng, J.~Zhu, Nuno: A general framework for learning parametric pdes with non-uniform data, arXiv preprint arXiv:2305.18694 (2023).

\bibitem{ovadia2023ditto}
O.~Ovadia, E.~Turkel, A.~Kahana, G.~E. Karniadakis, Ditto: Diffusion-inspired temporal transformer operator, arXiv preprint arXiv:2307.09072 (2023).

\bibitem{ovadia2023vito}
O.~Ovadia, A.~Kahana, P.~Stinis, E.~Turkel, G.~E. Karniadakis, Vito: Vision transformer-operator, arXiv preprint arXiv:2303.08891 (2023).

\bibitem{2023HemmasianFarimani}
A.~Hemmasian, A.~Barati~Farimani, Reduced-order modeling of fluid flows with transformers, Physics of Fluids 35~(5) (2023).

\bibitem{2023NIPSLiShuFarimani}
Z.~Li, D.~Shu, A.~B. Farimani, Scalable transformer for pde surrogate modeling, Advances in neural information processing systems 36 (2023).

\bibitem{zhu2023fourier}
M.~Zhu, S.~Feng, Y.~Lin, L.~Lu, Fourier-deeponet: Fourier-enhanced deep operator networks for full waveform inversion with improved accuracy, generalizability, and robustness, arXiv preprint arXiv:2305.17289 (2023).

\bibitem{guo2021construct}
R.~Guo, J.~Jiang, Construct deep neural networks based on direct sampling methods for solving electrical impedance tomography, SIAM Journal on Scientific Computing 43~(3) (2021) B678--B711.

\bibitem{guo2023transformer}
R.~Guo, S.~Cao, L.~Chen, \href{https://openreview.net/forum?id=HnlCZATopvr}{Transformer meets boundary value inverse problems}, in: The Eleventh International Conference on Learning Representations, 2023.
\newline\urlprefix\url{https://openreview.net/forum?id=HnlCZATopvr}

\bibitem{guo2023learn}
R.~Guo, J.~Jiang, Y.~Li, Learn an index operator by cnn for solving diffusive optical tomography: A deep direct sampling method, Journal of Scientific Computing 95~(1) (2023) 31.

\bibitem{mizera2023scattering}
S.~Mizera, Scattering with neural operators, arXiv preprint arXiv:2308.14789 (2023).

\bibitem{2022LiFarimani}
Z.~Li, A.~Barati~Farimani, Graph neural network-accelerated lagrangian fluid simulation, Computers \& Graphics 103 (2022) 201--211.

\bibitem{zhang2022hybrid}
E.~Zhang, A.~Kahana, E.~Turkel, R.~Ranade, J.~Pathak, G.~E. Karniadakis, A hybrid iterative numerical transferable solver (hints) for pdes based on deep operator network and relaxation methods, arXiv preprint arXiv:2208.13273 (2022).

\bibitem{wu2024capturing}
K.~Wu, X.-B. Yan, S.~Jin, Z.~Ma, Capturing the diffusive behavior of the multiscale linear transport equations by asymptotic-preserving convolutional deeponets, Computer Methods in Applied Mechanics and Engineering 418 (2024) 116531.

\bibitem{rahaman2018spectral}
N.~Rahaman, D.~Arpit, A.~Baratin, F.~Draxler, M.~Lin, F.~A. Hamprecht, Y.~Bengio, A.~Courville, On the spectral bias of deep neural networks, International Conference on Machine Learning (2019).

\bibitem{ronen2019the}
B.~{Ronen}, D.~{Jacobs}, Y.~{Kasten}, S.~{Kritchman}, The convergence rate of neural networks for learned functions of different frequencies, in: Advances in Neural Information Processing Systems, Vol.~32, 2019, pp. 4761--4771.

\bibitem{xu2020frequency}
Z.-Q.~J. {Xu}, Y.~{Zhang}, T.~{Luo}, Y.~{Xiao}, Z.~{Ma}, {Frequency Principle: Fourier Analysis Sheds Light on Deep Neural Networks}, Communications in Computational Physics 28~(5) (2020) 1746--1767.

\bibitem{kovachki2021universal}
N.~Kovachki, S.~Lanthaler, S.~Mishra, On universal approximation and error bounds for fourier neural operators, Journal of Machine Learning Research 22 (2021) Art--No.

\bibitem{zhao2022incremental}
J.~Zhao, R.~J. George, Y.~Zhang, Z.~Li, A.~Anandkumar, Incremental fourier neural operator, arXiv preprint arXiv:2211.15188 (2022).

\bibitem{Li_2020}
X.-A. Li, Z.-Q.~J. Xu, L.~Zhang, A multi-scale dnn algorithm for nonlinear elliptic equations with multiple scales, arXiv preprint arXiv:2009.14597 (2020).

\bibitem{WANG2021113938}
S.~Wang, H.~Wang, P.~Perdikaris, On the eigenvector bias of fourier feature networks: From regression to solving multi-scale pdes with physics-informed neural networks, Computer Methods in Applied Mechanics and Engineering 384 (2021) 113938.

\bibitem{li2021subspace}
X.-A. Li, Z.-Q.~J. Xu, L.~Zhang, Subspace decomposition based dnn algorithm for elliptic type multi-scale pdes, Journal of Computational Physics 488 (2023) 112242.

\bibitem{liu2021swin}
Z.~Liu, Y.~Lin, Y.~Cao, H.~Hu, Y.~Wei, Z.~Zhang, S.~Lin, B.~Guo, Swin transformer: Hierarchical vision transformer using shifted windows, in: Proceedings of the IEEE/CVF International Conference on Computer Vision, 2021, pp. 10012--10022.

\bibitem{Hackbusch2015}
W.~Hackbusch, Hierarchical matrices: algorithms and analysis, Springer, Berlin, 2015.

\bibitem{bhattacharya2021model}
K.~Bhattacharya, B.~Hosseini, N.~B. Kovachki, A.~M. Stuart, Model reduction and neural networks for parametric pdes, The SMAI journal of computational mathematics 7 (2021) 121--157.

\bibitem{zhang2021aggregating}
Z.~Zhang, H.~Zhang, L.~Zhao, T.~Chen, , S.~{\"O}. Arık, T.~Pfister, Nested hierarchical transformer: Towards accurate, data-efficient and interpretable visual understanding, in: AAAI Conference on Artificial Intelligence (AAAI), 2022.

\bibitem{ronneberger2015u}
O.~Ronneberger, P.~Fischer, T.~Brox, U-net: Convolutional networks for biomedical image segmentation, in: International Conference on Medical image computing and computer-assisted intervention, Springer, 2015, pp. 234--241.

\bibitem{choromanski2020rethinking}
K.~Choromanski, V.~Likhosherstov, D.~Dohan, X.~Song, A.~Gane, T.~Sarlos, P.~Hawkins, J.~Davis, A.~Mohiuddin, L.~Kaiser, et~al., Rethinking attention with performers, arXiv preprint arXiv:2009.14794 (2020).

\bibitem{linformer:2020}
S.~Wang, B.~Li, M.~Khabsa, H.~Fang, H.~Ma, Linformer: Self-attention with linear complexity, arXiv:2006.04768 (2020).

\bibitem{peng2021random}
H.~Peng, N.~Pappas, D.~Yogatama, R.~Schwartz, N.~A. Smith, L.~Kong, Random feature attention, arXiv preprint arXiv:2103.02143 (2021).

\bibitem{nguyen2021fmmformer}
T.~Nguyen, V.~Suliafu, S.~Osher, L.~Chen, B.~Wang, Fmmformer: Efficient and flexible transformer via decomposed near-field and far-field attention, Advances in neural information processing systems 34 (2021) 29449--29463.

\bibitem{xiong2021nystromformer}
Y.~Xiong, Z.~Zeng, R.~Chakraborty, M.~Tan, G.~Fung, Y.~Li, V.~Singh, Nystr{\"o}mformer: A nystr{\"o}m-based algorithm for approximating self-attention, in: Proceedings of the AAAI Conference on Artificial Intelligence, Vol.~35, 2021, pp. 14138--14148.

\bibitem{guadagnini1999nonlocal}
A.~Guadagnini, S.~P. Neuman, Nonlocal and localized analyses of conditional mean steady state flow in bounded, randomly nonuniform domains: 1. theory and computational approach, Water Resources Research 35~(10) (1999) 2999--3018.

\bibitem{gittelson2010stochastic}
C.~J. Gittelson, Stochastic galerkin discretization of the log-normal isotropic diffusion problem, Mathematical Models and Methods in Applied Sciences 20~(02) (2010) 237--263.

\bibitem{nelsen2021random}
N.~H. Nelsen, A.~M. Stuart, The random feature model for input-output maps between banach spaces, SIAM Journal on Scientific Computing 43~(5) (2021) A3212--A3243.

\bibitem{smith2019super}
L.~N. Smith, N.~Topin, Super-convergence: Very fast training of neural networks using large learning rates, in: Artificial intelligence and machine learning for multi-domain operations applications, Vol. 11006, SPIE, 2019, pp. 369--386.

\bibitem{ronen2019convergence}
B.~Ronen, D.~Jacobs, Y.~Kasten, S.~Kritchman, The convergence rate of neural networks for learned functions of different frequencies, Advances in Neural Information Processing Systems 32 (2019).

\bibitem{Matthew2020Fourier}
M.~{Tancik}, P.~P. {Srinivasan}, B.~{Mildenhall}, S.~{Fridovich-Keil}, N.~{Raghavan}, U.~{Singhal}, R.~{Ramamoorthi}, J.~T. {Barron}, R.~{Ng}, Fourier features let networks learn high frequency functions in low dimensional domains, arXiv preprint arXiv:2006.10739 (2020).

\bibitem{ashiqur2022u}
M.~A. Rahman, Z.~E. Ross, K.~Azizzadenesheli, U-no: U-shaped neural operators, arXiv e-prints (2022) arXiv--2204.

\bibitem{wang2020towards}
R.~Wang, K.~Kashinath, M.~Mustafa, A.~Albert, R.~Yu, Towards physics-informed deep learning for turbulent flow prediction, in: Proceedings of the 26th ACM SIGKDD International Conference on Knowledge Discovery \& Data Mining, 2020, pp. 1457--1466.

\bibitem{he2016deep}
K.~He, X.~Zhang, S.~Ren, J.~Sun, Deep residual learning for image recognition, in: Proceedings of the IEEE Conference on Computer Vision and Pattern Recognition, 2016, pp. 770--778.

\bibitem{deHoop2022cost}
M.~V.~d. Hoop, D.~Z. Huang, E.~Q. null, A.~M. Stuart, The cost-accuracy trade-off in operator learning with neural networks, Journal of Machine Learning 1~(3) (2022) 299–341.
\newblock \href {https://doi.org/10.4208/jml.220509} {\path{doi:10.4208/jml.220509}}.

\bibitem{tran2023factorized}
A.~Tran, A.~Mathews, L.~Xie, C.~S. Ong, \href{https://openreview.net/forum?id=tmIiMPl4IPa}{Factorized fourier neural operators}, in: The Eleventh International Conference on Learning Representations, 2023.
\newline\urlprefix\url{https://openreview.net/forum?id=tmIiMPl4IPa}

\bibitem{lanthaler2022error}
S.~Lanthaler, S.~Mishra, G.~E. Karniadakis, Error estimates for deeponets: A deep learning framework in infinite dimensions, Transactions of Mathematics and Its Applications 6~(1) (2022) tnac001.

\end{thebibliography}
\bibliographystyle{elsarticle-num}















\end{document}